%% file: ensembles_jmlr.tex
\documentclass[twoside,11pt]{article}

\usepackage{blindtext}

\usepackage{amsthm}
\usepackage{amsmath}
\usepackage{mathtools}

\usepackage{enumitem}

\usepackage{tikz}

\usepackage[abbrvbib]{jmlr2e}

\renewcommand{\epsilon}{\varepsilon}

\DeclareMathOperator{\Expec}{\mathbb{E}}%
\DeclareMathOperator{\Exps}{e}%
\DeclareMathOperator{\Proba}{\mathbb{P}}%

\newcommand{\abs}[1]{\left\lvert#1\right\rvert} %
\newcommand{\bernoulli}[1]{\mathcal{B}(#1)}%
\newcommand{\binomial}[1]{\mathrm{Bin}(#1)}%
\newcommand\ceil[1]{\lceil#1\rceil}
\newcommand{\Complex}{\mathbb{C}}%
\newcommand{\cov}[1]{\mathrm{Cov}\left(#1\right)}%
\newcommand{\defeq}{\vcentcolon =}%
\newcommand{\deter}[1]{\det \left(#1\right)}
\newcommand{\Diff}{\mathrm{d}}%
\newcommand{\eqdef}{=\vcentcolon }%
\newcommand{\expl}[1]{\mathrm{exp}\left(#1\right)}%
\newcommand{\expec}[1]{\Expec\left[#1\right]}%
\newcommand{\exps}[1]{\Exps^{#1}}%

\newcommand{\gaussian}[2]{\mathcal{N}(#1,#2)}
\newcommand{\ii}{\mathbf{i}}%
\newcommand{\inner}[2]{\langle #1,#2\rangle}%
\newcommand{\lambdamin}{\lambda_{\min}} 
\newcommand{\lambdamax}{\lambda_{\max}}
\newcommand{\norm}[1]{\left\lVert #1 \right\rVert}%
\newcommand{\Integers}{\mathbb{N}}%
\newcommand{\littleo}[1]{o\left(#1\right)}%
\newcommand{\proba}[1]{\Proba\left (#1\right )}%
\newcommand{\Reals}{\mathbb{R}}%
\newcommand{\Relint}{\mathbb{Z}}%
\newcommand{\spec}[1]{\mathrm{Spec}(#1)}%
\newcommand{\Xbar}{\overline{X}}%
\newcommand{\yhat}{\hat{y}}%
\newcommand{\yinf}{\ybar_{\infty}}%

\newcommand{\ybar}{\overline{y}}%

\newcommand{\review}[1]{{#1}}

\newtheorem{assumption}{Assumption}

\makeatletter
\def\th@plain{%
	\thm@notefont{}%
	\itshape %
}
\def\th@definition{%
	\thm@notefont{}%
	\normalfont %
}
\makeatother

\usepackage{lastpage}
\jmlrheading{26}{2025}{1-\pageref{LastPage}}{3/24; Revised
	6/25}{9/25}{24-0408}{Pierre-Alexandre Mattei and Damien Garreau}

\ShortHeadings{Are ensembles getting better all the time?}{Mattei and Garreau}
\firstpageno{1}

\begin{document}

\title{Are Ensembles Getting Better All the Time?}

\author{\name Pierre-Alexandre Mattei \email pierre-alexandre.mattei@inria.fr \\
        \addr Université Côte d'Azur\\
        Inria, Maasai team\\
        Laboratoire J.A. Dieudonné, CNRS \\
        Nice, France \\
       \name Damien Garreau \email damien.garreau@uni-wuerzburg.de \\
\addr Julius-Maximilians-Universit\"at W\"urzburg \\
Institute for Computer Science / CAIDAS \\
W\"urzburg, Germany
}

\editor{Maya Gupta}

\maketitle

\begin{abstract}%
Ensemble methods combine the predictions of several base models. 
We study whether or not including more models 
always improves their average performance. 
This question depends on the kind of ensemble considered, as well as the predictive metric chosen. 
We focus on situations where all members of the ensemble are \emph{a priori} expected to perform \review{equally well}, 
which is the case of several popular methods such as random forests or deep ensembles. 
In this setting, we show that ensembles are getting better all the time if, and only if, the considered loss function is convex. 
More precisely, 
in that case, the \review{loss} 
of the ensemble is a decreasing function of the number of models. 
When the loss function is nonconvex, we show a series of results that can be summarised as:
ensembles of good models keep getting better, and ensembles of bad models keep getting worse. 
To this end, we prove a new result on the monotonicity of tail probabilities that may be of independent interest. 
We illustrate our results on a 
medical 
problem (diagnosing melanomas using neural nets) and a ``wisdom of crowds'' experiment (guessing the ratings of upcoming movies).
\end{abstract}


\section{Introduction}

The idea that groups are collectively wiser than individuals 
can be traced back to various branches of science, ranging from political science to philosophy, economics, or sociology. 
In statistics, forecasting, and machine learning, this idea is embodied by \emph{ensemble methods}. 
Broadly speaking, ensembles combine the outputs of several base predictive models in a single decision, often by averaging (for regression) or by majority voting (for classification).

Early applications were mostly related to forecasting \citep{bates1969combination,clemen1989combining} and decision theory \citep{stone1961opinion,genest1986combining}. In the 1990s, a significant amount of new ensembling techniques were developed by the burgeoning machine learning community, including notably boosting \citep{freund1997decision}, Breiman's bagging \citeyearpar{breiman1996bagging} or random forests \citeyearpar{breiman2001random}, Ho's \citeyearpar{ho1998random} random subspace, neural network ensembles \citep{hansen1990neural,krogh1994neural}, and Bayesian model averaging \citep{hoeting1999bayesian}. 
Ensembles are routinely used in practice, for instance for probabilistic weather forecasting \citep{gneiting2005weather} (see also Kuncheva's \citeyearpar{kuncheva2014combining} monograph on the topic).
\citet{grinsztajn2022why} go as far as claiming that 
random forests and boosting remain the most efficient machine learning methods for tabular data. 
More recently, ensembles of neural networks have been instrumental within the deep learning revolution, notably through dropout \citep{srivastava2014dropout,baldi2014dropout,hron2018variational} and deep ensembles \citep{lakshminarayanan2017simple} and their variants (see, \emph{e.g.},  \citealp{sun2022towards,laurent2022packed}, and references therein).
Random projections ensembles are still actively researched \citep{cannings2017random,cannings2021random}.  

The empirical successes of ensembles seem to indicate that the more models are being aggregated, the better the ensemble is. 
\review{This, of course, can only be true on average. Adding a consistently bad model to the family will decrease the overall accuracy---for instance adding a model that always predicts the wrong class.}
The theoretical question we address in this paper is the following: \textbf{Is it always true that an ensemble of $K+1$ models performs better \review{on average} than an ensemble of $K$ models?}
\review{We refer to this question as the \emph{monotonicity question}.}
A positive answer \review{would have} practical consequences, \review{encouraging us to always take more base models in our ensemble. 
As we will see, the picture is more nuanced and needs precise analysis to emerge.}

\begin{figure}
\centering
\includegraphics[width = \columnwidth]{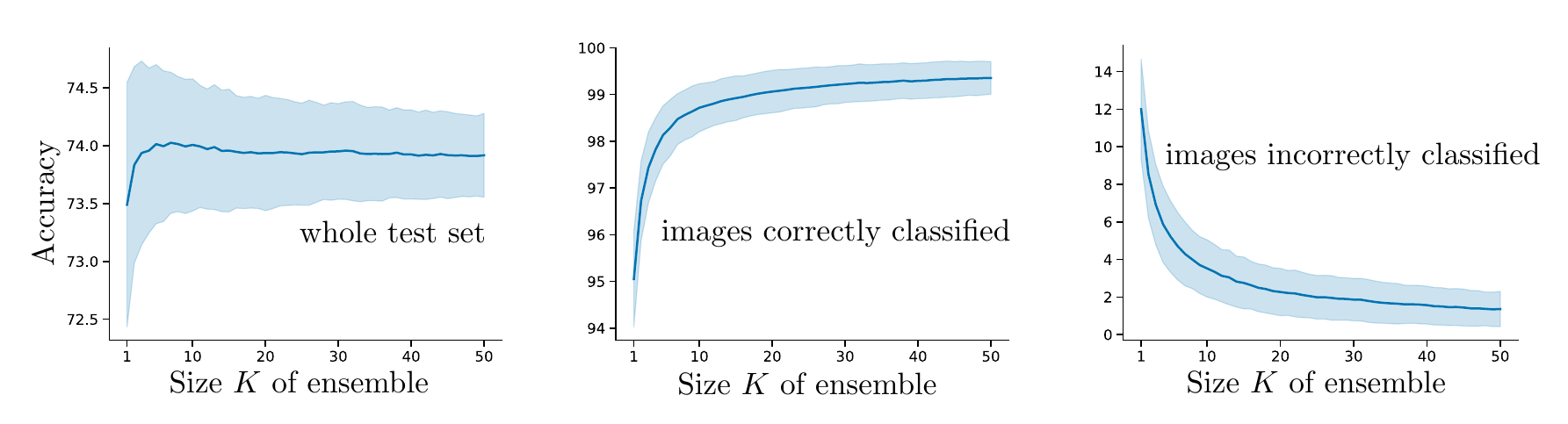}
\caption{\label{fig:accuracy_derma}Evolution of the accuracy of a dropout ensemble on the DermaMNIST dataset as the number of models grows (mean and standard deviation over $500$ repetitions). \emph{(Left)} The accuracy averaged over the whole test set has no clear monotonic pattern. \emph{(Middle)} Accuracy averaged over images for which the asymptotic prediction \review{(\emph{i.e.} the prediction of an ensemble of infinite size)} 
$\ybar_\infty$ is correct.
\review{The accuracy is getting better when the ensemble size grows.}  \emph{(Right)} Accuracy averaged over images for which $\ybar_\infty$ \review{is} incorrect.
\review{The accuracy is going down.} 
The behaviours of these middle and right panels \review{are related to the non-convexity of the classfication error}, and are explained by our theory. Notice that the three $y$-axes have different scales.}
\end{figure}

\subsection{A motivating example: neural nets for medical diagnosis}
\label{sec:intro-medical}

In this section, we consider the following practical problem: we want to predict whether a dermatoscopic image corresponds to a benign keratosis-like lesion or a melanoma. 
We use the DermaMNIST \citep{yang2023medmnist} data set, based on the HAM10000 collection \citep{tschandl2018ham10000}, and retain only the classes ``benign keratosis'' and ``melanoma'' (more details on the data set are available in Section~\ref{sec:experiments}).
Using a training set and a validation set, we train a LeNet-like convolutional neural network \citep{lecun1998gradient} 
with dropout to discriminate these two classes. 
We then use a Monte Carlo dropout ensemble to classify the test set \citep{gal2016dropout,hron2018variational}, and look at the evolution of the accuracy (Figure~\ref{fig:accuracy_derma}) and the cross-entropy loss (Figure~\ref{fig:crossent_derma}) as the ensemble grows. 
There is a crisp difference between the behaviours of the two losses. The overall accuracy has no clear monotonic pattern, but if we split the test set into two subsets depending on whether the asymptotic prediction is right or wrong, we notice that the accuracy is \review{in}creasing for right predictions and \review{de}creasing for wrong ones! 
On the other hand, the cross-entropy appears to be clearly decreasing, regardless of the quality of the prediction. 
Our goal is to provide clear theoretical explanations for these phenomena.

\begin{figure}
\centering
\includegraphics[width = \columnwidth]{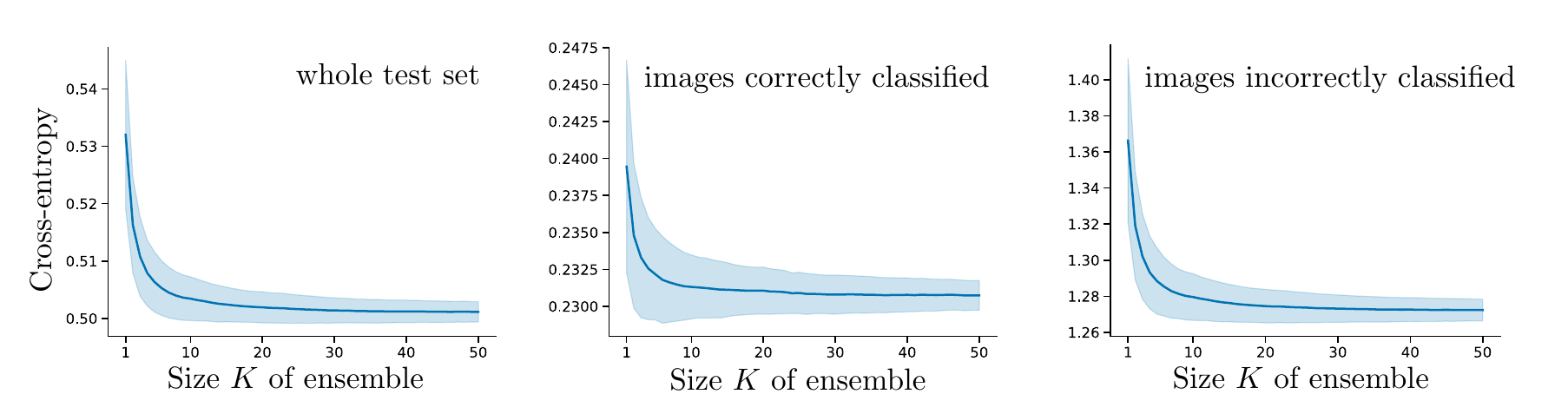}
\caption{\label{fig:crossent_derma}Evolution of the cross-entropy (or negative log-likelihood) of a dropout ensemble on the dermatology data set as the number of models grows (mean and standard deviation over 500 repetitions). \emph{(Left)} Cross-entropy averaged over the whole test set. \emph{(Middle)} Cross-entropy averaged over images for which the asymptotic prediction $\ybar_\infty$ is correct.  \emph{(Right)} Cross-entropy averaged over images for which the asymptotic prediction $\ybar_\infty$ in incorrect. As predicted by our theory, all curves are decreasing \review{because the cross-entropy loss is convex}. Notice that the three $y$-axes have different scales.}
\end{figure}

\subsection{Summary and organisation of the paper}

In this work, we focus on the simplest instance of ensembles, that involve unweighted averaging of predictions that are ``not too dependent.'' 
While simple, this setting remains the usual formalization of most ``wisdom of crowds'' studies, and the basis of state-of-the-art machine learning methods like random forests and deep ensembles. 
In this context, formalised in Section~\ref{sec:setting}, 
for an ensemble to be monotonically improving is deeply connected to the local geometry of the chosen loss function, more specifically to its convexity and its smoothness. 
More precisely, we show that:
\begin{itemize}
\item For \textbf{convex loss functions and exchangeable predictions} (Section~\ref{sec:convex-losses}), the expected loss is a non-increasing function of the number of ensembles members (Theorem~\ref{th:exch}), generalising several results previously obtained for specific losses \citep{noh2017regularizing,probst2018tune}. 
This result can be strengthened for strictly convex losses (Theorem~\ref{th:strong}).  
\item For \textbf{nonconvex loss functions and independent and identically distributed (i.i.d.) predictions} (Section~\ref{sec:non-convex-losses}), we show that, 
\review{when the number of base models is large enough,}
good ensembles keep getting better, and bad ensembles keep getting worse in the following sense:
\begin{itemize}
\item when the loss if sufficiently smooth, ensembles keep eventually getting better when the loss function is locally convex at their asymptotic predictions, whereas they keep getting worse if the loss is locally concave there (Theorem~\ref{th:smooth-ncvx});
\item for the classification error, ensembles keep eventually getting better when their asymptotic prediction corresponds to the true label, and keep getting worse if it corresponds to a wrong label (Theorem~\ref{th:class-error}). 
This result comes with not-so-mild regularity conditions that we relate to classical counterexamples \emph{à la} Condorcet.
\end{itemize}
\item To prove our result for nonsmooth losses, we prove two new theorems on the monotonicity of tail probabilities that may be of independent interest (Theorems~\ref{th:monotonicity-tail-proba} and~\ref{th:monotonicity-tail-proba-multivariate}).
This result is based on ``strong'' large deviations theorems \citep{bahadur1960deviations,petrov1965probabilities, joutard_2017}.
\end{itemize}
Informally, our results may be summarised by the sentence \textbf{ensembles are getting better all the time if, and only if, the considered loss function is convex}. 
We support these claims by two real data experiments presented in Section~\ref{sec:experiments}. 

\subsection{Related works: monotonicity and the wisdom of crowds}

The main heuristic justification of ensembles is that groups are collectively wiser than individuals, because aggregated decisions have a smaller variance than individual ones. This general idea is referred to in various, sometimes colorful, ways: some, {like \citet{dekel2009vox}}, praise the ``vox populi,'' following Galton's \citeyearpar{galton1907vox} seminal paper; the phrase ``wisdom of crowds,'' used for instance by \citet{prelec2017solution}  or \citet{simoiu2019studying}, was immensely popularised by Surowiecki's \citeyearpar{surowiecki2005wisdom} bestselling book; political scientists like \citet{althaus2003collective} often debate the virtues of the ``miracle of aggregation'' as a theoretical argument in favor of democracy; many scientists made the popular phrase ``two heads are better than one'' their own (\emph{e.g.}, \citealp{koriat2012two,chu2003question}). \review{A recent review of the philosophical ramifications of collective intelligence was provided by \citet{elkin2025opinion}.}

\cite{nicolas1785essai} proposed a first mathematical formalisation of this heuristic insight. His goal was to prove that, within a trial, the larger the number of jurors, the better. 
In modern machine learning jargon, he showed that ensembles containing an odd number of Bernoulli-distributed \review{base models} are monotonically improving, with respect to the classification error (see \citealt{shteingart2020majority}, for a modern proof). 
Surprisingly, he also noticed that monotonicity is broken if the ensemble is allowed to contain an even number of \review{models}.

\review{An important (and quite diverse) amount of work has been devoted to the theoretical study of ensembles; while giving a detailed overview is out of our scope, it is worth providing a few entry points to this literature. Some authors analysed the properties of specific kinds of ensembles, such as bagging \citep{larsen2023bagging}, random forests \citep{scornet2025theory}, random projections \citep{cannings2017random}, dropout \citep{hron2018variational} or deep ensembles \citep{wild2023rigorous}. Others studied the relationship between the benefits of ensembling and the diversity of the ensemble \citep{ortega2022diversity,theisen2023ensembles}. More related to our work are papers that study the behaviour of ensembles when $K$ is large. In particular, \citet{cannings2017random} and \citet{lopes2020estimating} derived asymptotic expansions of the test loss when $K \rightarrow \infty$. Unfortunately, such expansions do not capture monotonicity, even for large enough $K$ (this will be discussed further in Section \ref{sec:condorcet-example}).

However, in spite of this large body of work, the monotonicity question has been quite overlooked, especially theoretically. \citet{lam_et_al_1997}, \citet{berend2005monotonicity}, and \citet{dougherty2009odd} did revisit Condorcet's theorem, and studied the issues that make the presence of an even number of models problematic. Several generalisations of this theorem were also derived, in particular with weaker dependence assumptions (see, \emph{e.g.}, \citealp{ladha1993condorcet}). Closer to our investigations \citet{probst2018tune}, in the context of random forests, performed a very thorough theoretical and empirical study of the monotonicity question.} Their empirical conclusions were that, for some performance metrics (the cross-entropy or the Brier score for classification and the mean squared error for regression), ensembles are getting better all the time, but that is not necessarily the case for the classification error. Moreover, they provided theoretical support for some of these phenomena. 
Their paper was a key inspiration for this work, and our main contribution is to provide a unified view of their insights by essentially showing that ensembles are getting better all the time if, only if, the considered loss function is convex. 
Indeed, the cross-entropy, the Brier score, and the mean squared error are all convex losses (and always lead to monotonic improvements), whereas the classification error is not (and may lead to non-monotonicity). 

The relationship between the usefulness of ensembling and convexity was perhaps first highlighted by \citet{stone1961opinion} in the context of decision theory. 
It was further discussed by several papers in the 1990s \citep{mcnees1992uses,perrone1993improving,breiman1996bagging,george1999bayesian}. An interesting historical perspective on the question was provided by \citet[Section~6]{manski2010consensus}. However, the relationship between the monotonicity question and convexity appears to have been overlooked.

\paragraph{Monotonicity beyond ensembles}  
It is worth mentioning some related work on monotonicity in statistics and machine learning. A first intriguing phenomenon is the fact that the coverage probability of a confidence interval for the true parameter of a Bernoulli random variable is non-monotonic (see, \emph{e.g.}, \citealt{brown_et_al_2001}, Figure 1). 
In plain words, and perhaps counter-intuitively, increasing the number of observations does not guarantee that the probability for the confidence interval to contain the true parameter increases! 
Another related question is the one of the \emph{monotonicity of learning}: will a machine learning algorithm perform better when adding one new sample to the training set?
\review{\citet{devroye_gyorfi_lugosi_2013} call such class of models \emph{smart rules} predictors and notices that it is quite challenging to know for sure which models are ``smart.'' 
Indeed, most consistency results only provide an asymptotic statement in the sample size, with no monotonicity guarantee. 
More recently, this problem was cast as an open problem by \citet{viering2019open} at the 32nd Conference on Learning Theory (COLT). 
Progress on that front was reviewed by \citet{viering2022shape}.}
\review{Finally, the question of monotonicity also arises when looking at the variance of a model, as illustrated, for instance by \citet[Exercise~2.4]{rasmussen2006gaussian}, %
for Gaussian processes. 
}

\section{Background, notations, and examples}
\label{sec:setting}

\subsection{A general predictive framework}

We consider a predictive task where several methods output predictions $\hat{y}$, living in a convex set $C$. 
The quality of these predictions is evaluated using a loss function $L: C \rightarrow \mathbb{R}$. 
This broad setting contains several important cases, and we present below a few important examples. 
In most of these examples, $\hat{y}$ will be of the form $f(x)$, where $f$ is a \review{base model} and $x$ a given \review{fixed} data point. 
This means that $L(\yhat)$ is the loss of an \emph{individual prediction},  rather than the loss averaged over a dataset. 
Note that we require the prediction set to be convex because we want that averaged predictions remain inside $C$ (this is barely a restriction, since when the prediction set is nonconvex we can simply replace it by its convex hull). 

\textbf{Supervised learning} and \textbf{forecasting} are the main examples we are interested in. Consider a given data point $x \in \mathcal{X}$, with true label/response $y\in \mathcal{Y}$. The prediction will typically be $\hat{y} = f(x)$ where $f$ can be for instance a neural network, a random forest, or even a human expert opinion (\emph{e.g.}, a doctor diagnosing a patient based on their records $x$). 
In this setting, loss functions have the familiar form $L(\hat{y}) = \ell(\hat{y}, y)$ where $\ell: C \times \mathcal{Y} \rightarrow \mathbb{R}$ measures the mismatch between the true label and its prediction (see, \emph{e.g.}, \citealt{shalev2014understanding}).
As an example, for neural net classification with $n_\text{cl}$ classes, the prediction set $C$ will either be $\mathbb{R}^{n_\text{cl}}$ (in the absence of a final softmax layer, in this case $\hat{y}$ denotes the logits) or the ${n_\text{cl}}$-simplex ($\hat{y}$ will then be the output of the softmax layer). In both cases, $\ell$ will typically be the cross-entropy, defined in the case where $\hat{y}$ is the softmax output, as
\begin{equation}
L(\hat{y}) = \ell(\hat{y}, y) = \sum_{j=1}^{n_\text{cl}} -y_j \log \hat{y}_j
\, ,
\end{equation}
where the true label is represented by a one-hot encoding  $y= (y_1,\ldots,y_{n_\text{Cl}}) \in \{0,1\}^{n_\text{Cl}}$. 
Another common choice of loss function in that case is the classification error, also known as the $0-1$ loss. 
\review{Denoting $\text{argmax}(u)$ the index of the coordinate of $u$ with the largest value, the classification error can be written}
\begin{equation}
\label{eq:classif_error}
L(\hat{y}) = \ell(\hat{y}, y)  =\left\{\begin{matrix}
1 \; \text{ if } \; \text{argmax}(\hat{y}) \neq  \text{argmax}(y)\\ 
0 \; \text{ if } \; \text{argmax}(\hat{y}) =  \text{argmax}(y)
\, .
\end{matrix}\right.
\end{equation}

If the response set is continuous (\emph{e.g.}, if $\mathcal{Y} = \mathbb{R}$ as in a standard regression problem), the prediction set can be $\mathbb{R}$ if we are simply interested in predicting a point estimate of the response, then $\ell$ can be the squared error: $L(\hat{y}) = \ell(\hat{y}, y) = (\hat{y} - y)^2$.
If we are interested in the full conditional distribution of the response, $C$ can be defined as the set of probability densities over $\mathbb{R}$. 
The  loss $\ell$ can then be any scoring rule like the negative log-likelihood $- \log \hat{y} (y | x) $ or the continuous ranked probability score, see, \emph{e.g.}, \citet{gneiting2007strictly}. 
Other supervised examples with more exotic prediction spaces and loss functions include image segmentation or ordinal regression.%

 \textbf{Density estimation} is generally unsupervised, and not directly a predictive problem. 
 It is nevertheless a particular case of our framework. Here, $\hat{y}$ will be a probability density over the data space $\mathcal{X}$, for instance modelled by a mixture model or a deep generative model, which means that $C$ will be the set of all densities over $\mathcal{X}$. For a given $x \in \mathcal{X}$, a popular loss function is then the negative log-likelihood $L(\hat{y}) = - \log \hat{y}(x)$.  
 Similarly to the probabilistic regression example outlined above, any scoring rule can also be chosen (\emph{e.g.}, score matching and its variants). 
 More exotic losses include using statistical divergences between the distribution  $\hat{y}$ and the empirical distribution of a test or validation set. While ensembling is less common for density estimation than forecasting and supervised learning, it has been successfully used in several contexts. For instance, \citet{riesselman2018deep} used ensembles of deep generative models to model biological sequences and \citet{russell2015bayesian} and \citet{casa2021better} have used ensembles of Gaussian mixtures for density estimation and clustering.
 
\textbf{Statistical estimation} involves an unknown parameter $\theta^\star \in \Theta$. In this case, $C = \Theta$ and $\hat{y}$ can be any estimator of $\theta$.  A possible loss function in this example is given by $L(\hat{y}) = d( \hat{y}, \theta^\star)$, where $d$ is a distance over $\Theta$. Here again, the use of ensembles is less common than for supervised learning. Nonetheless, bagging and its variants have been thoroughly studied for statistical estimation (see, \emph{e.g.}, \citealt{buhlmann2002analyzing, friedman2007bagging,bach2008bolasso}).

\paragraph{Why is $L$ only a function of $\hat{y}$?} In many of these examples, the loss function depends on other quantities than the predictions (the  features, the true label...), yet we just denote it by $L(\yhat)$. 
Our analysis aims at studying the impact of the predictions on the loss, therefore we keep all these other quantities fixed, and vary only the predictions. This is why we only model our loss function $L$ as a function of the predictions $\hat{y}$. In particular, this means that our results typically concern the loss evaluated for a single given data point and a single ground truth label. Of course, we are also often interested in evaluating the loss not only on individual data points, but on batches or entire data sets. Our results on individual data points can then be extended by averaging them, as we discuss further in Section \ref{sec:whichresults?}.

\subsection{Several predictions, and how to average them}

From now on, we consider that we are given $K \in \mathbb{N}^\star$ different predictions, 
named $\hat{y}_1,\ldots,\hat{y}_K \in C$, that we wish to aggregate. 

\subsubsection{Where do our predictions \texorpdfstring{$\hat{y}_1,\ldots,\hat{y}_K$}{} come from?}
\label{sec:where_do_they_come_from}

A fundamental example to keep in mind is the situation where the predictions come from slight variations from the same model. 
In Breiman's \citeyearpar{breiman1996bagging} \emph{bagging}, each $\yhat_k$ comes from using the same algorithm on a different bootstraped version of the original data.
This is also the cornerstone of random forests, for which each $\yhat_k$ is the output of a single (random) decision tree \citep{breiman2001random}. 
Each tree has a different structure coming from the randomness of the tree construction procedure, leading to distinct, random $\yhat_k$s. 
Another important example, pioneered by \citet{hansen1990neural} and popularised by \citet{lakshminarayanan2017simple}, uses as $\yhat_k$s the outputs of neural networks trained with different weights initialisations. 
Another popular way of creating ensembles of deep nets is via \emph{dropout} \citep{srivastava2014dropout,baldi2014dropout}: each $\yhat_k$ is obtained by sampling a different dropout mask and applying it to the same backbone network. \citet{gal2016dropout} call such a technique Monte Carlo dropout.

In all the examples presented in the previous paragraph, the $\hat{y}_1,\ldots,\hat{y}_K$ are i.i.d. random variables. 
Assumptions of this sort will be required for our proofs.
Indeed, with the exception of Lemma~\ref{lem:triangle}, we will always require our predictions to be random variables with some (weak) dependence structure. 
We refer to Section~\ref{sec:exchangeable} for a simple example of what can go wrong when there is too much dependency between the individual models. 

We want to emphasise that our framework also encompasses situations where we are given a few fully deterministic predictions (for instance, the forecasts of several different physical models). One way to see them as random predictions and to make them amenable to our theory is to assume that the ordering of the $\hat{y}_1,\ldots,\hat{y}_K$ does not matter, or equivalently that we have no knowledge of how these models fare against each other. This can be formalised by assuming that we do not observe $\hat{y}_1,\ldots,\hat{y}_K$  but $\hat{y}_{\sigma(1)},\ldots,\hat{y}_{\sigma(K)}$, where $\sigma$ is a uniformly sampled permutation of $\{1,\ldots,K\}$. Because $\sigma$ is random, $\hat{y}_{\sigma(1)},\ldots,\hat{y}_{\sigma(K)}$ is also random. While not i.i.d. in general, this \emph{randomly reordered ensemble} satisfies another weaker condition called \emph{exchangeability}. A sequence of random variables is exchangeable when any of its reorderings has the same distribution, we will give a formal definition in Eq.~\eqref{eq:echnangeability}.

\subsubsection{How to aggregate predictions?}

Perhaps the simplest idea to aggregate predictions is to take their (unweighted) mean, namely to consider the prediction
\begin{equation}
\ybar_K =  \frac{1}{K} \sum_{k=1}^K \hat{y}_k
\, .
\end{equation}
 Note that we are allowed to consider such an aggregation since $C$ is convex, which implies that the ensemble prediction $\ybar_K$ indeed belongs to $C$. This simple aggregation scheme is used by many ensembles, for instance bagging (and in particular random forests), deep ensembles, or Monte Carlo dropout. 
 Many other aggregation schemes have been proposed (see, \emph{e.g.}, \citealt{kuncheva2014combining}, Chapters~4 and~5), but our analysis is fundamentally tied to the simple averaging technique, so we will not discuss them further.  
 A convenient property of $\ybar_K$ is that its asymptotic behaviour is quite intuitive. Indeed, when $\hat{y}_1,\ldots,\hat{y}_K$ are i.i.d. with mean $ \ybar_\infty = \mathbb{E}  \left[ \hat{y}_1 \right] = \cdots =  \mathbb{E}  \left[ \hat{y}_K \right]$, then the law of large numbers implies that $\ybar_K \longrightarrow \ybar_\infty$ when  $K \rightarrow +\infty$. 
 We will refer to $\ybar_\infty$ as the \emph{asymptotic prediction} of the ensemble.

Several other aggregation schemes can, in some settings, be equivalent to the simple average. For instance, the geometric mean can be written as a regular mean in log scale, since $\sqrt[K]{\hat{y}_1\cdots \hat{y}_K} = \exp \left( \frac{1}{K} \sum_{k=1}^K \log \hat{y}_k\right)$. The popular method of \emph{majority voting} can also be seen as a form of standard average.
Indeed, consider for instance a binary classification problem for which several votes $\hat{y}_1,\ldots,\hat{y}_K \in \{0,1 \}$ are cast for a true class $y \in \{0,1\}$. The prediction space $ \{0,1 \}$ is nonconvex, but we may choose to see the votes as elements of its convex hull $C = [0,1]$. 
A natural choice of loss function is then the classification error $L(\hat{y})$, defined by Eq.~\eqref{eq:classif_error}. 
The most natural way of combining these votes is to use majority voting, considering as final choice $\tilde{y}_\text{MV,$K$} = \mathbf{1}(\frac{1}{K} \sum_{k=1}^K\hat{y}_k>0.5)$. It turns out this majority vote will have the same loss as the ensemble prediction: $L(\tilde{y}_\text{MV,$K$} ) = L (\ybar_K ) $, which means that majority voting is equivalent to simple averaging.

\subsubsection{What sort of results are we going to prove?}
\label{sec:whichresults?}

A natural way of looking at the performance of the random ensemble prediction is to study its expectation. 
Our goal will be to prove that, under various settings, $\mathbb{E}[ L (\ybar_K)]$ is monotonic, where $\mathbb{E}$ is taken with respect to the randomness associated to the predictions. 
Sometimes, we will merely show that $\mathbb{E}[ L (\ybar_K)]$ is \emph{eventually monotonic}, which means that there exists an integer $K_0$ such that $(\mathbb{E}[ L (\ybar_K)])_{K \geq K_0}$ is monotonic.

A few comments about what this means are in order. 
The quantity $\mathbb{E}[ L (\ybar_K)]$ corresponds to the loss of a \emph{single} prediction problem, \emph{e.g.}, predicting the label of a single image, averaged over the randomness of the ensemble. This is in contrast with many machine learning results that average over a dataset. Nevertheless, the monotonicity of individual losses will imply monotonicity for losses averaged over datasets by simply applying our results to every single point of the dataset and using the fact that an average of decreasing (respectively increasing) sequences remains decreasing (respectively increasing). 
This is illustrated in Fig.~\ref{fig:crossent_derma}: our results imply that the cross-entropy loss decreases for every single test point, and in turn that the whole test cross-entropy decreases.
Things are slightly different in Fig.~\ref{fig:accuracy_derma}: we cannot average our results over the whole test set (since it is composed of points whose accuracy eventually grows and points whose accuracy eventually decreases). However, we can split the test dataset into two sub-datasets for which we can apply our theorems, leading to the middle and right panels of Fig.~\ref{fig:accuracy_derma}.

\section{Ensembles get better for convex losses}
\label{sec:convex-losses}

In this section, the loss function $L: C \rightarrow \mathbb{R}$ is always assumed to be convex. In a supervised context where $L(\hat{y}) = \ell(\hat{y}, y)$, this means that $\ell$ is convex in its first argument. In classification, this is for instance the case of the cross-entropy, the Brier score, or the hinge loss. In regression, convex losses include the mean absolute error, the mean squared error, or Huber's \citeyearpar{huber1964robust} robust loss. Examples in image segmentation include the cross-entropy or the Lov{\'a}sz-softmax loss \citep{berman2018lovasz}. Another ubiquitous example of convex loss is the negative log-likelihood, both for density estimation and probabilistic prediction. In decision theory, convex losses correspond to concave utilities and to risk aversion (see \emph{e.g.} \citealt{parmigiani2009decision}, Chapter 4). %

\subsection{Warm-up: An ensemble predicts better than a single model}

Let us begin by reviewing a few classical arguments in favour of ensembling with a convex loss, based on Jensen's inequality. The simple and compelling reasoning below, that dates back to early works on forecasting at least \citep[Section 3.4]{winkler1971probabilistic}, was called the ‘‘algebraic wisdom of crowds'' by \citet{manski2016interpreting}. Similar arguments were also given, for instance, by \citet{mcnees1992uses} or \citet[Section 4]{breiman1996bagging} in the specific cases of the squared and absolute losses and by \citet{perrone1993improving} and  \citet{george1999bayesian} for a generic convex loss. 
The convexity of $L$ directly implies that, for all $\yhat_1,\ldots,\yhat_K \in C$,
\begin{equation}
\label{eq:warmup-1sample} 
 L (\ybar_K) \leq
	  \frac{1}{K} \sum_{k=1}^K L(\hat{y}_k)  \leq \max_{k \in \{1,\ldots,K\} } L(\hat{y}_k)
   \, .
\end{equation}
Therefore the ensemble will perform better, on average, than any of its components randomly chosen (with uniform probability $1/K$). In particular, the ensemble will \emph{always} perform better than its worst component.

When all elements of the ensemble are identically distributed, so are $L(\hat{y}_1),\ldots, L(\hat{y}_K)$ and 
integrating the left-hand side of Eq.~\eqref{eq:warmup-1sample}
leads to
\begin{equation}
\label{eq:simple-jensen}
\mathbb{E}  \left[ L\left(\ybar_K \right) \right] \leq \mathbb{E}  \left[ L(\hat{y}_1) \right] = \cdots =  \mathbb{E}  \left[ L(\hat{y}_K) \right]
\, ,
\end{equation}
assuming that all involved expectations exist.
This means that an ensemble of identically distributed predictors performs better than a single one of them, on average. If we further assume that the ensemble has mean $ \ybar_\infty = \mathbb{E}  \left[ \hat{y}_1 \right] = \cdots =  \mathbb{E}  \left[ \hat{y}_K \right]  $, then 
\begin{equation}
	\label{eq:infinite}
	L(\ybar_\infty) \leq \mathbb{E}  \left[ L\left(\ybar_K \right) \right],
\end{equation}
due to Jensen's inequality and the fact that $\mathbb{E}[\hat{y}_k] = \ybar_\infty$ for all $k$. We chose to denote this mean $\ybar_\infty$ because the law of large numbers implies that $\ybar_K \longrightarrow \ybar_\infty$ almost surely when  $K \rightarrow \infty$, if we further assume that the $\yhat_ks$ are independent. 
According to Eq.~\eqref{eq:infinite}, this ``infinite" ensemble that would be obtained by averaging over all possible predictions, is better than any of its finite counterparts. However, this does not say how these finite ensembles fare against each other. As we shall see now, this question can also be addressed by playing around with Jensen's inequality.

\subsection{A first monotonicity lemma}

All the new results in this section are based on a simple identity, that expresses $\ybar_K$ as a convex combination of smaller ensembles. The starting point is to notice the following equality between $K$-dimensional vectors:
\begin{align}
	\begin{bmatrix}
		1 \\
		1 \\
		\vdots \\ 1\\
		1
	\end{bmatrix} &= \frac{1}{K-1} \left(\begin{bmatrix}
		0 \\
		1 \\
		\vdots \\ 1\\
		1
	\end{bmatrix} + \begin{bmatrix}
		1 \\
		0 \\
		\vdots \\ 1\\
		1
	\end{bmatrix} + \cdots +
	\begin{bmatrix}
		1 \\
		1 \\
		\vdots \\ 0\\
		1
	\end{bmatrix} +
	\begin{bmatrix}
		1 \\
		1 \\
		\vdots \\ 1\\
		0
	\end{bmatrix} \right).
\end{align}
This directly leads to
\begin{equation}
	\label{eq:triangle}
	\frac{1}{K} \sum_{k = 1}^{K} \hat{y}_k = \frac{1}{K} \sum_{j = 1}^{K} \frac{1}{K-1}  \sum_{k \neq j}  \hat{y}_k
\, ,
\end{equation}
which, combined with the convexity of $L$, leads to the following lemma.

\begin{lemma}[monotonicity lemma]
\label{lem:triangle}
Assume that $L$ is a convex function. 
Then, for any $\hat{y}_1,\ldots,\hat{y}_K \in C$, we have
\begin{equation} 
L\left( \frac{1}{K} \sum_{k = 1}^{K} \hat{y}_k \right) \leq \frac{1}{K} \sum_{j = 1}^{K} L\left(\frac{1}{K-1}  \sum_{k \neq j}  \hat{y}_k\right)
\, .
\end{equation}
\end{lemma}

This inequality may already be interpreted as a monotonicity result: removing an element of the ensemble uniformly at random will hurt, on average. Note that this comes \emph{without any assumption on the ensemble itself}, which could be composed of very different and potentially deterministic models. 
To prove stronger results, we need to make additional assumptions on the ensemble. 
Our assumptions will be probabilistic, \review{either requiring independent, identically-distributed predictions, or a weakened version of this assumption called exchangeability}.
\review{They} will translate the fact that we suspect that, \emph{a priori}, all elements of the ensembles perform \review{equally well}.

\subsection{Exchangeable ensembles are getting better all the time}
\label{sec:exchangeable}

We saw in the previous section (Eq.~\ref{eq:infinite}) that assuming that the ensemble is identically distributed led to an asymptotic form of monotonicity: the infinite ensemble $\ybar_\infty$ performs better than any finite ensemble, which in turn performs better than any single model, because of Eq.~\eqref{eq:simple-jensen}. 
However, the identically distributed assumption alone will not take us very far. 
Indeed, let us consider the following example: 
an ensemble of three models such that $\hat{y}_1$ and $\hat{y}_2$ are independent and identically distributed, and $\hat{y}_3 = \hat{y}_1$. 
Then, one can show that the ensemble gets \emph{worse} in average when adding its third component:
\begin{equation}
\label{eq:contrex-id}
\mathbb{E} \left[ L\left( \frac{\hat{y}_1+\hat{y}_2}{2} \right) \right] \leq \mathbb{E} \left[ L\left( \frac{\hat{y}_1+\hat{y}_2 + \hat{y}_3}{3} \right) \right]
\, .
\end{equation}
This directly follows from the identity
\begin{equation}
	\frac{\hat{y}_1+\hat{y}_2}{2} = \frac{1}{2} \left(   \frac{2 \hat{y}_1 + \hat{y}_2}{3}  +\frac{\hat{y}_1 + 2\hat{y}_2}{3} \right),
\end{equation}
combined with the Jensen's inequality and the fact that
\begin{equation}
	\hat{y}_1+\hat{y}_2+\hat{y}_3 = 2 \hat{y}_1 + \hat{y}_2 =_d  2 \hat{y}_2 + \hat{y}_1,
\end{equation}
where $=_d$ denotes equality in distribution.

This simple counter-example motivates the use of a stronger assumption than identically distributed. A natural choice would be to assume that the elements of the ensembles are i.i.d., which is often the case in practice (for instance in random forests or deep ensembles). We will rely on an assumption that is actually weaker than i.i.d. but stronger than i.d.: \emph{exchangeability}. We remind that a sequence of random variables $\hat{y}_1,\ldots,\hat{y}_K$ is exchangeable when, for any permutation $\sigma \in {S}_K$ of the indices,
    \begin{equation}
     \label{eq:echnangeability}   (\hat{y}_1,\ldots,\hat{y}_K) =_d (\hat{y}_{\sigma(1)},\ldots,\hat{y}_{\sigma(K)}).
    \end{equation}
For more insights 
on exchangeability, we recommend \citet{diaconis1988recent}. 
Intuitively, an ensemble is exchangeable when the ordering of its members is not important. 

Many popular ensembles are exchangeable. This is of course the case of all i.i.d. ensembles (including bagging, random forests, random projections, deep ensembles), but there also exists non-i.i.d. exchangeable ensembles. One example that was sketched at the end of Section~\ref{sec:where_do_they_come_from} is the one of a randomly reordered ensemble. Another general example is the one of cross-validated committees (\citealp{rokach2010ensemble}, Section~2.2.4) that train different models on different exchangeables subsets of the training data. Political scientists have also argued that exchangeability is a more sensible assumption than i.i.d. for jury theorems \emph{à la} Condorcet \citep{ladha1993condorcet,peleg2012extending}.

However, there are also ensembles for which the exchangeability assumption is inherently wrong. 
This is notably the case of boosting, or any similar technique that trains models iteratively, making the construction of the $k$th base model dependent on the $k-1$ previous models, and hereby making their ordering essential. 
This is also the case when we have prior knowledge that some models are much better than others. In these two cases, using a weighted average instead of just the usual mean $\ybar_K$ is more advisable.

Under this assumption of exchangeability, we can show the following result, that has a simple interpretation: \emph{when the loss is convex and the ordering does not matter, ensembles monotonically get better, on average}.

\begin{theorem}[monotonicity for convex losses, \citealp{marshall1965inequality}]
\label{th:exch}
	If the loss $L$ is convex and $\hat{y}_1,\ldots,\hat{y}_K$ are exchangeable, then,
	\begin{equation}
		\label{eq:exch_mono}
		\mathbb{E} \left[ L\left( \ybar_K \right)   \right] \leq \mathbb{E} \left[ L\left( \ybar_{K-1} \right)   \right].
	\end{equation}
\end{theorem}

\begin{proof}
	From Lemma \ref{lem:triangle}, we have, by integrating
	\begin{equation} 
		\label{eq:proof_eq}
		\mathbb{E} \left[ L\left( \frac{1}{K} \sum_{k = 1}^{K} \hat{y}_k \right) \right] \leq \frac{1}{K} \sum_{j = 1}^{K} \mathbb{E} \left[L\left(\frac{1}{K-1}  \sum_{k \neq j}  \hat{y}_k\right)\right]
  \, .
	\end{equation}
But, because $\yhat_1,\ldots,\yhat_K$ are exchangeable, for any $j \in \{1,\ldots, K\}$,
	\begin{equation}
		\mathbb{E} \left[	L\left(\frac{1}{K-1}  \sum_{k \neq j}  \hat{y}_k\right)\right] = \mathbb{E} \left[L\left( \frac{1}{K-1} \sum_{k = 1}^{K-1} \hat{y}_k \right) \right]
  \, ,
	\end{equation}
thus all the terms of the sum of the right-hand side of Eq.~\eqref{eq:proof_eq} are equal to the right-hand side of  Eq.~\eqref{eq:exch_mono}, which concludes the proof.
\end{proof}

A version of Theorem~\ref{th:exch} was first proved by \citet{marshall1965inequality} in a very different context, using the concept of \emph{majorisation} (see, \emph{e.g.},  \citealt{marshall2011inequalities}). In the machine learning community, a similar result was independently rediscovered and popularised in the context of variational inference for deep generative models by \citet{burda2016}. The connection between majorisation and the results of \citet{marshall1965inequality} and \citet{burda2016} was highlighted and discussed by \citet{mattei2022}. \citet{struski2022bounding} provided an alternative proof of the result of \citet{burda2016} that is similar to our short proof of Theorem~\ref{th:exch}.

None of the works mentioned in the previous paragraph are related to ensembles. However, in the context of ensembles, several particular cases of Theorem~\ref{th:exch} for specific loss functions have been previously proved. In particular,  \citet{probst2018tune} dealt with two cases: when $L$ is the square loss and when $L$ is the cross entropy and $K$ is large enough. Moreover, \citet{noh2017regularizing} noticed that the result of \citet{burda2016} could be applied to ensembles when $L$ is the negative log-likelihood.

\subsection{Strict improvements for strongly convex losses}

The previous result show that the average loss is non-increasing. 
Under which conditions is it \emph{strictly} decreasing? 
This depends both on the properties of the loss (for instance, if $L$ is constant, no ensemble can strictly improve) and of the ensemble (an ensemble containing identical copies of the same predictor will have constant loss). We will see now that two simple assumptions are sufficient to ensure strict improvements:  strong convexity of the loss, and existence of a nonsingular covariance of the ensemble. 
In this subsection, we further assume that the prediction set $C$ is equipped with a norm $\| \cdot \|$. 

\begin{definition}[strong convexity]
Let $\mu > 0$.	
A function $L: C \rightarrow \mathbb{R}$ is $\mu$-strongly convex when, for all $x,y \in C$, we have
\begin{equation}
L(tx+(1-t)y) \leq tL(x) +(1-t)L(y) - \mu t(1-t) \|x - y \|^2
\, .
\end{equation}
\end{definition}

Strong convexity leads to a strengthening of Jensen's inequality (see, \emph{e.g.}, \citealt{nikodem2014strongly}, Theorem~2) that we can combine with our key identity Eq.~\eqref{eq:triangle} to prove strict improvements for i.i.d. ensembles.

\begin{theorem}[strict monotonicity for strongly convex losses] \label{th:strong}
Assume that the ensemble is exchangeable and has second order moments.
Let us write $\Sigma \succ 0$ for its covariance matrix. 
Define $\rho \defeq \textup{tr}(\textup{Cov}(\yhat_1,\yhat_2))/ \textup{tr}(\Sigma)$ its correlation coefficient.
Then, if the loss $L$ is $\mu$-strongly convex, then
	\begin{equation}		\label{eq:strong}
		\mathbb{E} \left[ L\left( \ybar_K \right)   \right] \leq \mathbb{E} \left[ L\left( \ybar_{K-1}\right)   \right] -  \frac{\mu}{K(K-1)} (1- \rho) \textup{tr} (\Sigma).
	\end{equation}
\end{theorem}

The proof of Theorem~\ref{th:strong} can be found in Appendix~\ref{app:strong}. 
For i.i.d. ensembles like random forests of deep ensembles, we have $\rho = 0$ which slightly simplifies the bound. 
This bound is tight when $L$ is the square loss and $\hat{y}_1, \ldots, \hat{y}_K$ are i.i.d. standard Gaussian. 
When $\yhat_1 = \cdots = \yhat_K$, $\rho = 1$, and the bound is also tight. 

This result formalises the fact that we can expect the loss improvements of ensembling to be larger as the variance of the ensemble grows. One may wonder what happens in the extreme case where the variance is infinite. Perhaps surprisingly, nothing clear can be said in this case. 
Indeed, take for instance $\hat{y}_1,\ldots,\hat{y}_K$ to be i.i.d. Cauchy random variables. Then, $\ybar_K$ is also Cauchy distributed, which means that the performance of the ensemble will be constant for any loss function, even though the variance is infinite.

\section{What happens for nonconvex losses?}
\label{sec:non-convex-losses}

For convex losses, we saw that the theory is quite simple and compelling: the average performance of ensembles is always improving. 
The picture becomes much more complex when the loss is no longer assumed to be convex. 
Indeed, we saw in Section~\ref{sec:intro-medical} that there was a crisp difference between the cross-entropy loss and the accuracy.

As we will see, there is nonetheless a general pattern for nonconvex losses. Some ensembles get better: they are the ones that, on average, ``get it right" (in the sense that $\ybar_\infty$ is a ``good prediction"), and correspond for instance to the ones on the middle panel of Fig.~\ref{fig:accuracy_derma}. On the other hand, when  $\ybar_\infty$ is a ``bad prediction,'' we will see that ensembling hurts (as it does on the right panel of Fig.~\ref{fig:accuracy_derma}). 

There are two sorts of nonconvex loss that will interest us. The first ones are the sufficiently smooth ones. Their monotonicity will be analysed by Taylor-expanding the loss around $\ybar_\infty$. We will then study the most popular nonconvex loss: the classification error, that will prove challenging to tackle due to its discontinuity. 
Unlike results for the convex case, which were true for any $K$, all results in this section will be asymptotic, in the sense that we will prove that $\mathbb{E}[ L(\ybar_K)]$ is \emph{eventually monotonic}  (\emph{i.e.}, that it is monotonic for $K$ large enough).

\subsection{Smooth nonconvex losses}
\label{sec:smoothnoncvx}

A first case of nonconvex loss that is easy to tackle is the one of a \emph{concave} loss. 
Indeed, if $L$ is concave, then $-L$ is convex which, by virtue of Theorem~\ref{th:exch}, implies that $\mathbb{E}[-L(\ybar_K)]$ decreases and thus that $\mathbb{E}[L(\ybar_K)]$ increases: for a concave loss, ensembles are getting worse all the time!
However, to the best of our knowledge, concave losses are never used in statistics and machine learning. 
\citet{laan2017rescuing} considered the example of a ``square-root cost" $L(\hat{y}) = \sqrt{ |y - \hat{y} |}$ which is locally concave everywhere except at zero, but not \emph{globally} concave (which prevents us from using Theorem~\ref{th:exch}).

Although globally concave loss are not used, \emph{some losses are locally concave in some part of the prediction space}, and locally convex on another. When these losses are smooth enough, these concave/convex parts can be recognised by checking if the Hessian of the loss is negative/positive definite. 

\begin{figure}[t!]
\centering
\includegraphics[width = 0.98\columnwidth]{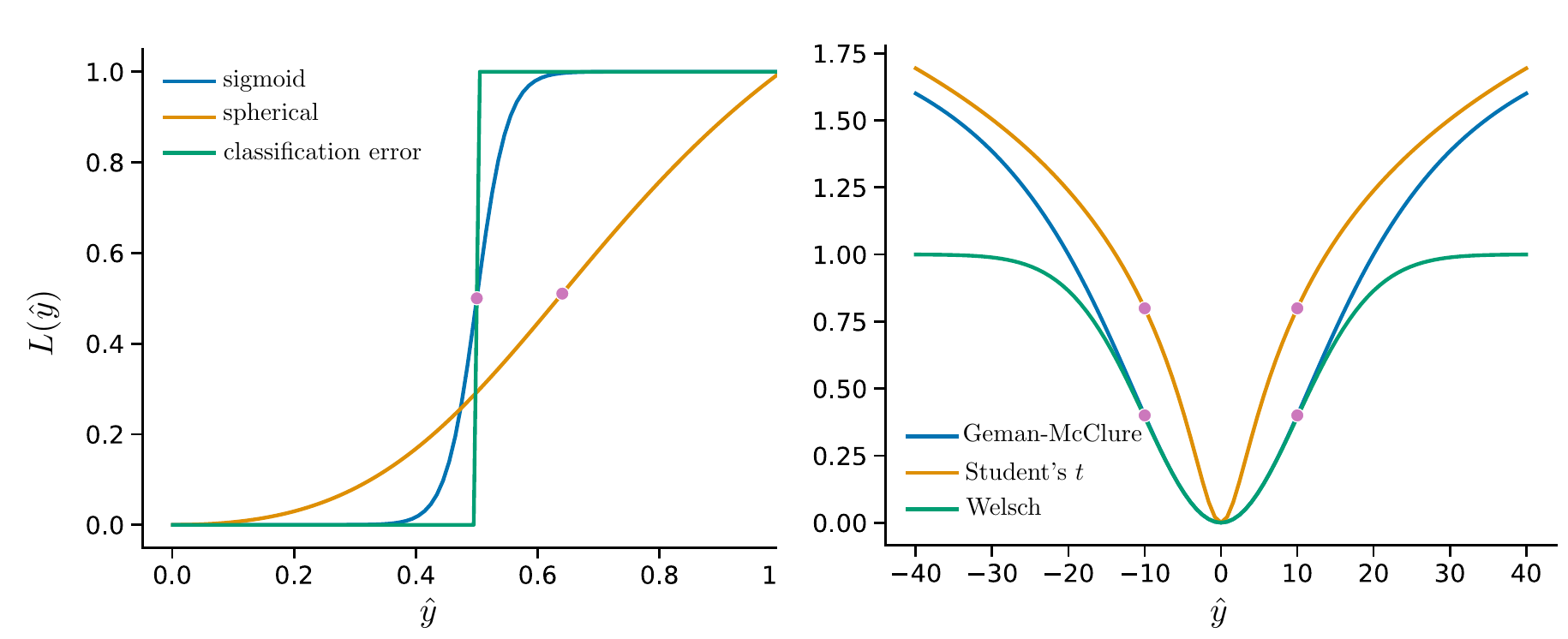}
\caption{\label{fig:sigmoid}Nonconvex loss functions for binary classification (left panel) and regression (right panel). In both case the true label/response is $y=0$. %
For smooth losses, purple dots denote inflexion points: the frontier between the concave part of the loss (that corresponds to ``right" predictions) and the concave one (that corresponds to ``wrong" ones).}
\end{figure}

Most of these smooth nonconvex loss functions share the informal property that \emph{the loss is locally convex in parts of the prediction space where predictions are ``right," and locally concave in parts where they are ``wrong"}. We have remained so far voluntarily vague regarding what we meant by ``right" and  ``wrong," because these terms will have different precise meanings depending on the loss functions. 
To illustrate this general behaviour, let us look at a few examples.

\subsubsection{Examples of smooth nonconvex losses}

We begin with losses designed for classification and consider for simplicity the binary case. We assume that the correct class is $y = 0$, and the prediction can be either an estimated probability $\yhat \in [0,1]$ that $1$ is the correct class, or a real-valued score $\yhat \in \mathbb{R}$ that is higher when the classifier is more confident that the label is $1$. We describe below three recipes for designing nonconvex smooth losses, and plot examples in Fig.~\ref{fig:sigmoid}.
\begin{itemize}
\item \textbf{Smooth approximations of the classification error} have been used as surrogates of it, and are generally nonconvex. An important instance is the \emph{sigmoid loss}. This sigmoid loss was popularised by its use within the DOOM II algorithm of  \citet{mason1999boosting}. It was also notably studied by \citet[Example~7]{bartlett2006convexity}, who showed it was classification-calibrated. This loss is easier to interpret in the probabilistic prediction setting, when $\yhat \in [0,1]$. Indeed, in this case, the sigmoid loss is locally convex when $\yhat < 0.5$, \emph{i.e.}, when the prediction using a 50\% threshold is correct, and locally concave when $\yhat > 0.5$, \emph{i.e.} when the prediction using the same threshold is wrong.  The \emph{Savage loss} introduced in a boosting context by \citet{masnadi2008design} and the \emph{normalised cross entropy} used by \citet{ma2020normalized} to train deep nets with noisy labels have a similar sigmoidal shape. They also share the same property of being convex when $\yhat < 0.5$, and concave when $\yhat > 0.5$. For these losses, a ``right prediction" is one whose most probable class is the correct one. In an opposite fashion, a``wrong  prediction" gives the highest probability to the wrong class.
\item \textbf{Scoring rules} are the most common way to design losses in the probabilistic setting (when $\yhat \in [0,1]$). Most strictly proper scoring rules for classification are convex \review{(\emph{e.g.} the cross-entropy or the Brier score)}, with the notable exception of the \emph{spherical score} (see, \emph{e.g.}, \citealp{gneiting2007strictly}). Like the sigmoid loss, the spherical score is convex when $\yhat$ is small and concave when $\yhat$ is large. 
However, unlike the sigmoid whose inflection point was at $\yhat = 50\%$, the inflection point of \review{the spherical score} is at $\yhat = 1/8 + \sqrt{17}/8 \approx 64\%$.
We give a formal definition in Appendix~\ref{app:spherical}. 
For the spherical loss, a ``right prediction" is thus one that will predict the true label $y = 0$ using the somewhat peculiar cutoff of  $64\%$ instead of the usual $50\%$. This illustrates the asymmetry of the spherical score.
\item One way to avoid overfitting is to use a loss function that makes the model pay a price for overconfident predictions (potentially \emph{even correct ones}). 
This is the rationale behind the \textbf{tangent loss}, introduced by \citet{masnadi2010design} as a boosting objective. 
Since this is not a probabilistic loss, predictions are real-valued scores $\yhat \in \mathbb{R}$. The tangent loss is concave when $\yhat$ is either too large or to small, penalising both overconfident mistakes and successes. In this case,  ``right predictions" correspond to the ones that are not overconfident. \review{For more details on the tangent loss, see Appendix \ref{app:spherical}.}
\end{itemize}

Regarding regression, the most commonly occurring nonconvex smooth losses are used to induce robustness to outliers. The main motivation for such losses is that the squared loss will typically make the model pay an extremely large price for large errors.
A way to fix this behavior is to design losses that resemble the squared error when errors are small, but grow less steadily when errors are large. 
The most famous robust loss is Huber's \citeyearpar{huber1964robust}, which is equal to the squared error when the error is below a threshold, and grows linearly when the error is above this threshold. Huber's loss is convex, therefore the analysis of the previous section guarantees that ensembles will get better all the time. However, many other losses that follow this rationale are nonconvex. A nice overview of these robust regression losses was recently offered by \citet{barron2019general}, who introduced a new general loss that generalises many previous losses (both convex and nonconvex). All nonconvex losses encompassed by Barron's generalisation (as well as others, such as using a Student's $t$ likelihood) share the property of being convex when the error $(\hat{y} - y)^2$ is smaller than a prespecified threshold $c$, and nonconvex when the error is larger than $c$. For all these nonconvex losses ``bad predictions" are therefore outliers and ``good" ones correspond to inliers. 
A few of these robust losses are displayed on \review{the right-hand side of} Figure~\ref{fig:sigmoid}, all of them were chosen in order to have a cutoff of $c = 10$.

\subsubsection{Monotonicity of smooth nonconvex losses}

Motivated by the previous examples, we will study a general loss that is smooth, strictly convex in some part of the prediction set $C$ (in that part, the loss Hessian will be positive definite), and strictly concave in another part of $C$ (in that part, the loss Hessian will be negative definite). 
The following result shows that ensembles will eventually keep getting better when the average prediction $\ybar_\infty$ ends up in this ``good" part (\emph{i.e.}, where $L$ is strictly convex), and will keep getting worse when it ends up in the ``bad" part (\emph{i.e.} where $L$ is strictly concave).

\begin{theorem}[Monotonicity of smooth nonconvex losses]
\label{th:smooth-ncvx}
	Let $\yhat_1,\ldots,\yhat_K \in C$ be nondegenerate i.i.d. random variables whose first 5 moments are finite, and $L$ be a function with continuous and bounded partial derivatives of order up to $5$, with Hessian $\nabla^2 L = H$. Then
	\begin{enumerate}
		\item If $H(\ybar_\infty) \succ 0$, then the ensemble is eventually getting better: for $K$ large enough,
		\begin{equation} 
			\mathbb{E} \left[ L\left( \ybar_K \right) \right] < \mathbb{E} \left[L\left(\ybar_{K-1}\right)\right],
		\end{equation}
		\item If $H(\ybar_\infty) \prec 0$, then the ensemble is eventually getting worse: for $K$ large enough,
		\begin{equation} 
			\mathbb{E} \left[ L\left( \ybar_{K} \right) \right] >  \mathbb{E} \left[L\left(\ybar_{K-1} \right)\right].
		\end{equation}
	\end{enumerate}
\end{theorem}
The proof of Theorem~\ref{th:smooth-ncvx} is available in Appendix \ref{sec:proof_smooth}, and essentially relies on a fourth order Taylor expansion of the loss. As we will explain in Section \ref{sec:condorcet-example}, low-order expansions are not sufficient to prove monotonicity.

\paragraph{Does this result say something about the classification error?} As illustrated in Fig.~\ref{fig:sigmoid}, the sigmoid loss can be seen as a smooth approximation of the classification error. Thus, we could expect that the classification error will behave in the same way: ensembles whose prediction are eventually correct keep getting better, and ensembles who are eventually wrong keep getting worse. One might even contemplate the idea of an easy proof using the fact that the classification error is a limiting case of a sigmoid loss. Unfortunately, we did not manage to make such a proof work as the behavior of the classification error is considerably more subtle, and involves more complex regularity conditions than the smoothness and moment assumptions of Theorem~\ref{th:smooth-ncvx}.

\subsection{Why is the classification error difficult to handle?}
\label{sec:condorcet-example}

\subsubsection{Monotonicity is a higher-order phenomenon} 

Since our goal is merely to show that monotonicity eventually holds, it seems natural to hope that standard asymptotic results about ensembles could do the trick. For instance, \citet{cannings2017random} and \citet{lopes2020estimating} give conditions under which 
\begin{equation}
\label{eq:lopes-cannings}
    \mathbb{E} \left[ L\left( \ybar_K \right) \right] = \mathbb{E} \left[ L\left( \yinf \right) \right] + \frac{c}{K} + o\left(\frac{1}{K}\right),
\end{equation}
where $L$ is the binary classification error averaged over the test set, and $c$ is a constant. Unfortunately, it is possible to satisfy Eq.~\eqref{eq:lopes-cannings} and to not be eventually monotonic (consider the sequence $1/K + (-1)^K/K^2$). 
With such expansions, monotonicity can only be captured by inspecting higher-order terms. This also explains why a fourth-order Taylor expansion is needed in the proof of Theorem~\ref{th:smooth-ncvx}. This means that we will need quite precise asymptotics to tackle the classification error.

\subsubsection{A deceivingly simple counter-example} 

To understand why the classification error is much more difficult to study than smooth losses, it is interesting to look at the following (counter) example, that goes back to \cite{nicolas1785essai}, and was studied thoroughly by \citet{lam_et_al_1997}.

Consider a binary classification task, where the true label is $0$, and our predictions $\yhat_1,\cdots,\yhat_K$ are i.i.d. samples from a Bernoulli distribution $\mathcal{B}(\ybar_\infty)$. In that case, the classification error of any $\yhat \in [0,1]$ is simply $L(\yhat) = \mathbf{1}(\yhat \geq  0.5)$. The expected loss of the ensemble will therefore be
$	\mathbb{E}[L(\ybar_k)] = \mathbb{P}(\ybar_K \geq 0.5)$. It seems reasonable to conjecture that $\mathbb{P}(\ybar_K \geq 0.5)$ will go to zero decreasingly when $\ybar_\infty < 0.5$, and to one increasingly when $\ybar_\infty > 0.5$, which would be consistent with the analogy to the sigmoid loss sketched above. Unfortunately, this is not the case, and $\mathbb{P}(\ybar_K \geq 0.5)$ is not, in general, monotonic (even asymptotically).

Indeed, let us focus for instance on the case where the ensemble is eventually right, \emph{i.e.}, when $\ybar_\infty < 0.5$. 
In that case, it can be shown that the subsequence of ensembles containing an odd number of models $L(\ybar_{2K+1} )$ decreases (see \citealp{shteingart2020majority}, Theorem~2, for a proof). However, monotonicity is broken when we include even numbers of models  \citep{lam_et_al_1997}.
A heuristic explanation of this failure, already mentioned by \cite{nicolas1785essai}, is that \review{even} numbers of predictions allow ties, while \review{odd} numbers do not. 
If we alter slightly the problem to always allow ties by allowing jurors to abstain from predicting, then monotonicity eventually holds, as we discuss in the next section and in Appendix \ref{sec:limitations-monotonicity}.
\citet[Section 3.1.1]{probst2018tune} discuss this problem of ties as well, and possible tie-breakers.

This counterexample shows that an example as benign as a Bernoulli distribution is problematic for the classification error, and that a more complex mathematical machinery than the one of the previous proofs is required. We will develop such a machinery in the next section, based on the following simple reformulation of the expected classification error.

\subsubsection{A useful formulation of the expected classification error}
\label{sec:notations-classif}

We consider a classification problem with $n_\text{Cl}$ classes. 
The true label, seen as a 1-hot encoding, is 
$y \in \{0,1\}^{n_\text{Cl}}$, and a prediction is a $n_\text{Cl}$-dimensional vector $\yhat = (\hat{y}^{(c)})_{c \leq n_\text{Cl}} \in \mathbb{R}^{n_\text{Cl}}$ of scores (that can be, for instance, estimated probabilities of each class, or a discrete one-hot class prediction). Without loss of generality, we assume that the true class is the first one ($y=(1,0,\ldots,0)$). 
The classification error  is then equal to one whenever the score of the true class is smaller than the maximum score:
\begin{equation}
	\label{eq:classif_error2}
	L(\hat{y}) = \left\{\begin{matrix}
		1 \; \text{ if } \; \hat{y}^{(1)} \leq  \text{max}_{c \neq 1 } \left(  \hat{y}^{(c)} \right)\\ 
		0 \; \text{ if } \;\hat{y}^{(1)} >  \text{max}_{c \neq 1 } \left(  \hat{y}^{(c)} \right).
	\end{matrix}\right.
\end{equation}
The expected loss of an ensemble will thus be $ \mathbb{E}[L(\ybar_K)] = \mathbb{P}\left( {\ybar_K}^{(1)} \leq  \text{max}_{c \neq y}   {\ybar_K}^{(c)}\right)$, which is not a particularly appealing expression due to the presence of the maximum. However, it can be rewritten as a simple tail probability the following way:

\begin{lemma}[accuracy as margin tail probability]
\label{lem:representation}
For all $k \in \{1, \ldots, K\}$, let $X_k = \left( {\yhat_k}^{(1)} - {\yhat_k}^{(2)}, \ldots,  {\yhat_k}^{(1)} - {\yhat_k}^{(n_\text{Cl})}\right) \in  \mathbb{R}^{n_\text{Cl} -1} $
\review{and $\Xbar_K = \left( {\ybar_k}^{(1)} - {\ybar_k}^{(2)}, \ldots,  {\ybar_k}^{(1)} - {\ybar_k}^{(n_\text{Cl})}\right) \in  \mathbb{R}^{n_\text{Cl} -1}$.}
Then,
    \begin{equation}
    \mathbb{E}[L(\ybar_K)] = 1- \mathbb{P}\left( \Xbar_K > 0 \right) =  \mathbb{P}\left( -\Xbar_K \geq 0 \right)
\, ,
    \end{equation}
\review{where the inequality inside the probability is to be taken entry-wise.}
\end{lemma}

\begin{proof}
    \begin{align*}
	\mathbb{E}[L(\ybar_K)] &=1-  \mathbb{P}\left( {\ybar_K}^{(1)}  >  \max_{c \neq 1} {\ybar_K}^{(c)}\right)  \\ &=  1-  \mathbb{P}\left( {\ybar_K}^{(1)}  >   {\ybar_K}^{(2)}, \ldots, {\ybar_K}^{(1)}  >   {\ybar_K}^{(n_\text{Cl})} \right) 
	\\ &= 1- \mathbb{P}\left( \Xbar_K > 0 \right)
 \, ,
\end{align*} 
by definition of $\Xbar_K$. 
\end{proof}
The vector $X_k \in  \mathbb{R}^{n_\text{Cl} -1}$ corresponds to a \emph{margin}: it stores the differences between the score of the correct label and the scores of the incorrect ones. We will discuss this further in Section~\ref{sec:classif_error_enfin}.

Now that we have written the classification error as a simple tail probability, it seems natural to attack it using standard probabilistic asymptotics. 
As mentioned before, low-order asymptotics are not sufficient, and we will leverage ``strong'' large deviation theorems, that can be viewed as higher-order large deviations. 
\review{This analysis is carried out in the next
section, which is largely independent from the rest of the paper. 
We note that it may be of independent interest, beyond ensembles.
}

\subsection{Monotonicity of tail probabilities}
\label{sec:monotonicity-tail-probabilities}

In this section, we present two results on the monotonicity of tail probabilities of empirical means. 
We state these results in a general form, since they may be of broader interest. Namely, we consider a sequence of i.i.d. random variables $X_1,\ldots,X_n$ living in $\Reals^D$ with finite expectation $\mu\in\Reals^D$. As we hinted before, and will detail in Section \ref{sec:classif_error_enfin}, in the context of ensembles, $X_1,\ldots,X_n$ will be margins (the difference between the score of the correct class and the score of the incorrect ones). For now, we see them as generic random variables to keep this section independent from the rest of the paper.

Our goal is to find under which conditions $\proba{\Xbar_n \geq \mu + \varepsilon}$ is eventually strictly decreasing, for a given $\varepsilon > 0$. We know that $\proba{\Xbar_n \geq \mu + \varepsilon}$ converges to zero without additional assumptions (this is a consequence of the weak law of large numbers). We will first deal with the univariate case (Theorem~\ref{th:monotonicity-tail-proba}), for which more intuitive regularity conditions are available than for the general, multivariate setting (Theorem~\ref{th:monotonicity-tail-proba-multivariate}).
In the next section, we will be able to readily apply these results to the sequence of predictions $\yhat_k$ to obtain an analogous of Theorem~\ref{th:smooth-ncvx} for the classification error. 

\begin{theorem}[monotonicity of tail probabilities, univariate]
\label{th:monotonicity-tail-proba}
Let $X_1,\ldots,X_n$ be i.i.d. random variables with finite expectation $\mu$, and let $\varepsilon>0$. 
Assume furthermore that 
\begin{enumerate}[label=(\arabic*)]
\item $\expec{\exps{tX_1}}<+\infty$ for all $t\in \Reals$;
\item $\proba{X_1 > \mu + \epsilon} > 0$;
\item $X_1$ is absolutely continuous with respect to the Lebesgue measure;
\end{enumerate}
or, alternatively to (3), 
\begin{enumerate}[label=(\arabic*)]
\item[(3bis)] $X_1$ is a lattice random variable and $\proba{X_1=\mu + \epsilon}>0$.
\end{enumerate}
Then, $\proba{\Xbar_n \geq \mu + \varepsilon}$ and $\proba{\Xbar_n > \mu + \varepsilon}$ are both strictly decreasing for all $n$ large enough. 
\end{theorem}

The idea of the proof (which can be found in Section~\ref{sec:proof-monotonicity-tail-proba} of the Appendix) is to use precise asymptotic results on tail probabilities obtained by \citet{petrov1965probabilities} (see also \citealp{bahadur1960deviations}). These results are sometimes referred to as “sharp" or “strong" large deviation theorems (because they are more precise than standard large deviations results like Cramer's celebrated theorem).

We want to emphasize that Theorem~\ref{th:monotonicity-tail-proba}, while intuitive, may not hold if $X_1$ does not meet the assumptions outlined in the statement of Theorem~\ref{th:monotonicity-tail-proba}. 
As a first example, let us assume that $X_1$ has a $\alpha$-stable distribution with $\alpha <1$. 
Then one can show that $\proba{\Xbar_n > \mu + \varepsilon}$ is actually \emph{increasing} (see Section~\ref{sec:limitations-monotonicity} of the Appendix).
Of course, taking $\alpha < 1$ in this example prevents $X_1$ from having finite expectation. 
Another intriguing example, sketched in the previous subsection, is when $X_1\sim \bernoulli{\mu}$ with $\mu \in (0,1)$. 
In that case, one can also derive the exact distribution of $X_1+\cdots+X_n$, a binomial $\binomial{n,\mu}$. 
Thus $\proba{\Xbar_n > \mu + \varepsilon}$ can be written as a sum of binomial coefficients, whose behavior is notoriously difficult to fully comprehend (see {e.g.} \citet{lam_et_al_1997}, who show non-monotonicity in some cases).
Here it is assumption (3bis) which does not hold: there is no mass at $\mu+\epsilon$. As we illustrate in Appendix \ref{sec:limitations-monotonicity}, adding a small mass at $\mu+\epsilon$ restores monotonicity. Coming back to the counter example of Section \ref{sec:condorcet-example}, which corresponds to the case $\mu + \varepsilon = 0.5$, this means that allowing models to make no decision (i.e. outputting $\yhat = 0.5$ with nonzero probability) leads to monotonicity in the classical Condorcet setting. Intuitively, this can be interpreted as follows: allowing blank votes make ties possibles both when there is an odd or an even number of voters. It is quite interesting to see here the classical regularity conditions of \citet{bahadur1960deviations} and \citet{petrov1965probabilities} explain neatly the mathematical origins of the problem of even voters in Condorcet's theorem.

\medskip

Now we present a multivariate extension of Theorem~\ref{th:monotonicity-tail-proba}. 
Note that, from now on, when $u$ and $v$ are both vectors, $u\leq v$ means that $u_i\leq v_i$ for all indices. 

\begin{theorem}[Monotonicity of tails, multivariate]
\label{th:monotonicity-tail-proba-multivariate}
Let $X_1,\ldots,X_n\in\Reals^D$ be i.i.d. random vectors distributed as $X$ \review{with} finite mean $\mu$, moment-generating function $\phi:t\mapsto \expec{\expl{\inner{t}{X_1}}}$, and cumulant-generating function $\varphi =  \log\phi$. 
Assume that 
\begin{enumerate}[label=(\arabic*)]
\item $X$ has \review{strictly positive} density $\rho$ with respect to the Lebesgue measure on $\Reals^D$;
\item there exists $\alpha >0$ such that $\phi$ is differentiable and bounded on $U_\alpha$, the open ball of radius~$\alpha$;
    \item there exists constants $0< \lambdamin < \lambdamax < +\infty$ such that $\spec{\nabla^2 \varphi}\subseteq [\lambdamin,\lambdamax]$.
\end{enumerate}
Then, for any $\epsilon \in \Reals^D$ such that $\epsilon > 0$ and
   \begin{equation}
\label{eq:condition-epsilon}
    \min_{1\leq i \leq D} \frac{\epsilon_i}{\norm{\epsilon}} < \sqrt{\frac{\lambdamin}{\lambdamax}}
    \, ,
\end{equation}
it holds that $\proba{\Xbar_n \geq \mu + \epsilon}$ is decreasing for $n$ large enough. 
\end{theorem}

Theorem~\ref{th:monotonicity-tail-proba-multivariate} is proved in Appendix~\ref{sec:proof-monotonicity-tail-proba-multivariate}. 
In some sense, it generalises Theorem~\ref{th:monotonicity-tail-proba}, since dealing with multivariate summands (which will allow us to treat the multiclass classification error in the next section). 
However, this comes at the cost of more stringent assumptions on the distribution of $X$, and in particular the local behavior of its cumulative distribution function. Moreover, while Theorem~\ref{th:monotonicity-tail-proba} was also applicable to some lattice random variables, this multivariate Theorem~is restricted to absolutely continuous distributions. In particular, this means that $\proba{\Xbar_n > \mu + \epsilon} = \proba{\Xbar_n \geq \mu + \epsilon}$ is also eventually decreasing.

\review{

\begin{remark}[Monotonicity of tails under Gaussian assumption]
\label{rk:multivariate-gaussian}
Let us first notice that the result of Theorem~\ref{th:monotonicity-tail-proba-multivariate} is trivially true under the assumption that $X_1,\ldots,X_n$ are i.i.d. $\gaussian{0}{\Sigma}$. 
Indeed, in this simplified setting, 
\[
\proba{\Xbar_n \geq \mu + \epsilon} = \proba{\gaussian{0}{\Sigma} \geq \sqrt{n}\epsilon}
\, ,
\]
which is decreasing since we integrate with respect to a positive density. 

Nevertheless, let us check the assumptions of Theorem~\ref{th:monotonicity-tail-proba-multivariate} in the multivariate Gaussian setting in order to show that they are not too stringent. 
First, the density of the multivariate Gaussian does not cancel, ensuring that Assumption~(1) is met. 
Second, it is known that 
\[
\forall t\in\Reals^D, \qquad 
\phi_X(t) = \exp \left(\mu^\top t + \frac{1}{2}t^\top \Sigma t\right) 
\, ,
\]
which is bounded and differentiable on any open ball of finite radius, yielding Assumption~(2). 
Third, we observe that $\nabla^2 \phi = \Sigma$, a positive-definite matrix. 
Henceforth Assumption~(3) is satisfied. 

Then remains the question of the existence of $\epsilon\in\Reals^D$ such that $\epsilon > 0$ and Eq.~\eqref{eq:condition-epsilon} holds---after all, this condition could be impossible to meet, even in such a simple scenario. 
We see that it is not the case: taking $\lambdamin=\lambdamax = 1$ gives rise to the condition $\min_i \epsilon_i^2 < \norm{\epsilon}^2$, which is true for any $\epsilon\in\Reals^D$ provided that $D\geq 2$. 
To see this, one can reason by contradiction and suppose that $\min_i \epsilon_i \geq \norm{\epsilon}$, meaning that for all $i\in [D]$, $ \epsilon_i^2 \geq  \norm{\epsilon}^2$. 
Summing these inequalities returns $\norm{\epsilon}^2 \geq D\norm{\epsilon}^2$, a contradiction. 
A similar proof shows that any ratio $\lambdamin / \lambdamax \geq 1 / D$ does not actually put a constraint on $\epsilon$. 

When the $\lambdamin / \lambdamax$ is smaller, the condition over $\epsilon$ becomes more stringent, but the set of $\epsilon$ remains non-empty. 
Indeed Eq.~\eqref{eq:condition-epsilon} one can always take one of the coordinates of $\epsilon$ to be arbitrarily small with respect to the others, automatically satisfying the condition. 
For instance, one can consider $\epsilon$ such that $\epsilon_1=\sqrt{\lambdamin}$ and all other coordinates equal to $\sqrt{\lambdamax}$. 
Moreover, we note that Eq.~\eqref{eq:condition-epsilon} is a cone condition, meaning that all (positive) multiples of an $\epsilon$ satisfying the condition satisfies it itself. 
In definitive, Theorem~\ref{th:monotonicity-tail-proba-multivariate} holds true for any multivariate Gaussian distribution, with a set of $\epsilon$ being more and more constrained, but never empty. 
\end{remark}

}

\subsection{Ensembles and the classification error}

\label{sec:classif_error_enfin}

We are finally in position to deal with the classification error. We remind the notations introduced in Section \ref{sec:notations-classif}. We deal with $n_\text{Cl}$ classes, and the true label, seen as a 1-hot encoding, is 
$y \in \{0,1\}^{n_\text{Cl}}$. A prediction is a $n_\text{Cl}$-dimensional vector $\yhat = (\hat{y}^{(c)})_{c \leq n_\text{Cl}} \in \mathbb{R}^{n_\text{Cl}}$ of scores. Without loss of generality, we assume that the true class is the first one: $y=(1,0,\ldots,0)$. Our goal is to use the results of the previous section together with the following representation, obtained in Lemma \ref{lem:representation}:

    \begin{equation}
    \mathbb{E}[L(\ybar_K)] = 1- \mathbb{P}\left( \Xbar_K > 0 \right) =  \mathbb{P}\left( -\Xbar_K \geq 0 \right),
    \end{equation}
where $X_k = \left( {\yhat_k}^{(1)} - {\yhat_k}^{(2)}, \ldots,  {\yhat_k}^{(1)} - {\yhat_k}^{(n_\text{Cl})}\right) \in  \mathbb{R}^{n_\text{Cl} -1} $.

In the binary case, when $n_\text{Cl} = 2 $, $X_k = {\yhat_k}^{(1)}  - {\yhat_k}^{(2)} $ just corresponds to the \emph{margin}, that is, the difference between the score of the correct label and the score of the incorrect one. In the multiclass case, $X_k$ is a $(n_\text{Cl}-1)$-dimensional \emph{vector of margins} obtained for each possible ‘‘1 versus $c$" classification problems, for $c \in \{2,\ldots,n_\text{Cl}\}$. Following the seminal work of \citet{bartlett1998boosting}, margins have been used a lot to study the theory of ensembles (see, \emph{e.g.}, \citealp{biggs2022margins}, and references therein).

Before we state our theorem, we have to translate to the ensemble framework the regularity conditions of the two theorems of Section \ref{sec:monotonicity-tail-probabilities}. We will provide two sets of assumptions that will allow us to distinguish between ‘‘good" ensembles (that eventually improve) and ‘‘bad" ones (that eventually worsen). In the rest of this section, we assume that $\yhat_1,\ldots, \yhat_k \in \mathbb{R}^{n_\text{Cl}}$ are i.i.d. random variables with finite mean $\yinf$.

\begin{assumption}[Correct prediction]
\label{assumption:correct}
Let $\varepsilon =  \mathbb{E}[X_1] = {\ybar_\infty}^{(1)}  1_{n_\text{Cl} -1} - {\ybar_\infty}^{(-1)}$, \review{where we set $u^{(-1)}=(u_2,\ldots,u_p)^\top$ for any vector of $\Reals^p$.} 
There are either $n_\text{Cl} = 2$ classes and the regularity conditions of Theorem~\ref{th:monotonicity-tail-proba} are verified, or there are $n_\text{Cl} > 2$ classes and the regularity conditions of Theorem~\ref{th:monotonicity-tail-proba-multivariate} are verified, for $-X_1,\ldots ,-X_K$ and $\varepsilon$.
\end{assumption}

Assumption~\ref{assumption:correct} implies in particular that $\varepsilon =  \mathbb{E}[X_1] > 0$ which is equivalent to the fact that the expected margin $ {\ybar_\infty}^{(1)}  - \max_{c \neq 1} {\ybar_\infty}^{(c)} $ is positive. 
This assumption of positive margin is often used to study ensembles (see, \emph{e.g.}, \citealp{breiman2001random}).
\review{
It implies, in particular, that the asymptotic prediction is correct, that is, $L(\ybar_\infty) = 0$. 
}

\begin{assumption}[Completely incorrect prediction]
\label{assumption:incorrect}
	Let $\varepsilon = - \mathbb{E}[X_1] =   {\ybar_\infty}^{(-1)}  - {\ybar_\infty}^{(1)}  1_{n_\text{Cl} -1} $. There are either $n_\text{Cl} = 2$ classes and the regularity conditions of Theorem~\ref{th:monotonicity-tail-proba} are verified, or there are $n_\text{Cl} > 2$ classes and the regularity conditions of Theorem~\ref{th:monotonicity-tail-proba-multivariate} are verified, for $X_1,\ldots, X_K$ and $\varepsilon$.
\end{assumption}
Assumption~\ref{assumption:incorrect} implies in particular that $\varepsilon = -  \mathbb{E}[X_1] > 0$, which means that all margins are negative, \emph{i.e.}, that the model asymptotically gives the lowest score to the true class. Because of this, we call this an assumption of completely incorrect prediction. This implies in particular that $L(\ybar_\infty) = 1$.

\begin{theorem}[Monotonicity of the classification error]
\label{th:class-error}
	Let $\yhat_1,\ldots,\yhat_K \in \mathbb{R}^{n_\text{Cl}}$ be i.i.d. random variables. Then
	\begin{enumerate}
		\item If the prediction is asymptotically correct (Assumption \ref{assumption:correct}), then the ensemble is eventually getting better: for $K$ large enough,
		\begin{equation} 
			\mathbb{E} \left[ L\left( \ybar_K \right) \right] < \mathbb{E} \left[L\left(\ybar_{K-1}\right)\right],
		\end{equation}
		\item If the prediction is asymptotically completely incorrect (Assumption \ref{assumption:incorrect}), then the ensemble is eventually getting worse: for $K$ large enough,
		\begin{equation} 
			\mathbb{E} \left[ L\left( \ybar_{K} \right) \right] >  \mathbb{E} \left[L\left(\ybar_{K-1} \right)\right].
		\end{equation}
	\end{enumerate}
\end{theorem}

\begin{proof}
We begin by assuming that the prediction is correct (Assumption \ref{assumption:correct}). 
Then, Theorem~\ref{th:monotonicity-tail-proba} (when $n_\text{Cl} = 2$) or Theorem~\ref{th:monotonicity-tail-proba-multivariate} (when $n_\text{Cl} > 2$) ensures that $\proba{-\Xbar_n \geq  -\mathbb{E}[X_1] + \varepsilon} = \proba{-\Xbar_n \geq 0}$ decreases for $K$ large enough. 
Using Lemma~\ref{lem:representation}, this implies that $L(\ybar_K)$
eventually decreases.

We now deal with the case of an incorrect prediction (Assumption \ref{assumption:incorrect}). Again, Theorem~\ref{th:monotonicity-tail-proba} or Theorem~\ref{th:monotonicity-tail-proba-multivariate} ensures that $\proba{\Xbar_n >  \mathbb{E}[X_1] + \varepsilon} = \proba{\Xbar_n >  0}$ eventually decreases. In turn, 
$L(\ybar_K)$ eventually increases.
\end{proof}

\section{Experiments} 
\label{sec:experiments}

In this section, we present simple illustrations of our results on real data. The large-scale empirical study of \citet{probst2018tune} can also be seen as a compelling illustration of our results in the context of random forests, as their conclusions are in line with ours.

\subsection{Classification of skin lesions with convolutional networks}

We start by detailing the motivating example of Section \ref{sec:intro-medical}. We want to predict whether a dermatoscopic image corresponds to a benign keratosis-like lesion or a melanoma. We use the DermaMNIST \citep{yang2023medmnist} data set, based on the HAM10000 collection \citep{tschandl2018ham10000}, and retain only the classes ``benign keratosis" and ``melanoma'', because they are roughly balanced within DermaMNIST (the whole dataset is very unbalanced) and notoriously difficult to distinguish \citep{grant1999misdiagnosis}. The training/validation/test split is the same as the one from \citet{yang2023medmnist}, and consists of $1,548$ color images of resolution $28 \times 28$ for training, $221$ similar images for validation, and $443$ test images. 

\begin{figure}
\centering
\includegraphics[width = \columnwidth]{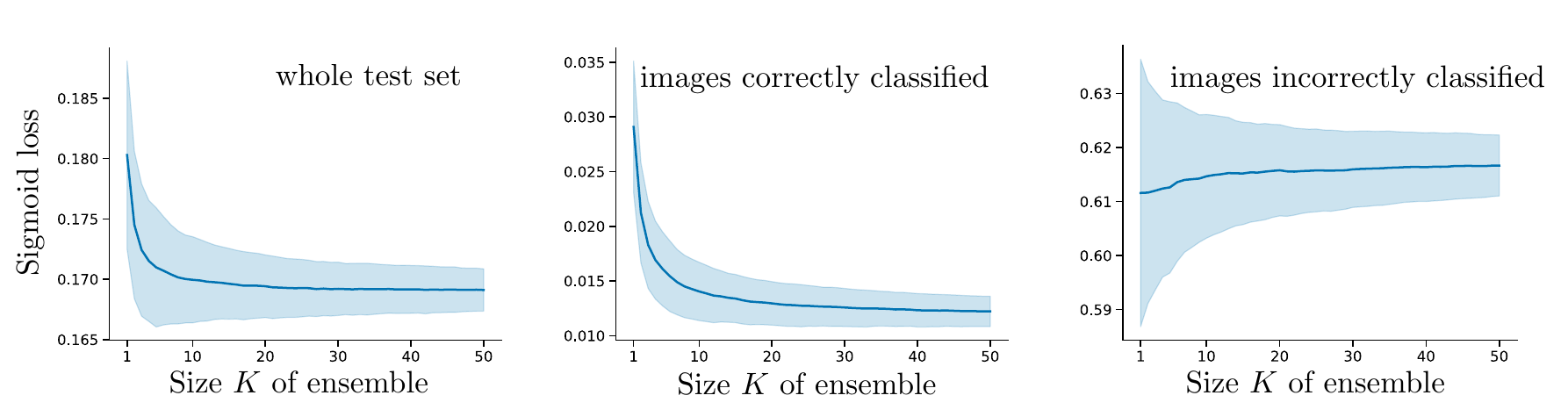}
\caption{\label{fig:sigmo_derma}Evolution of the sigmoid loss of a dropout ensemble on the 
dermatology data set as the number of models grows (mean and standard deviation over $500$ repetitions). \emph{(Left)} The sigmoid loss averaged over the whole test appears to be decreasing. \emph{(Middle)} The loss averaged over images for which the asymptotic prediction $\ybar_\infty$ is correct is getting better.  \emph{(Right)} The loss averaged over images for which the asymptotic prediction $\ybar_\infty$ in incorrect is getting worse. The behaviours of these middle and right panels are explained by our theory. Notice that the three $y$-axes have different scales.}
\end{figure}

We use a simple LeNet-like convolutional network \citep{lecun1998gradient} whose fully connected layers are regularised with a dropout rate of $50 \%$ \citep{srivastava2014dropout}. At test time, our base predictions $\yhat_1, \ldots ,\yhat_K$ are the outputs of the trained network with different random dropout masks. The ensemble prediction $\ybar_K$ is called Monte Carlo model averaging by \citet[Section 7.5]{srivastava2014dropout} and was popularised in particular by \citet{gal2016dropout}, who called it  \emph{Monte Carlo dropout}.

As we saw in the introduction (Figure~\ref{fig:crossent_derma}), the test cross-entropy is a decreasing function of $K$. This is explained by our theory, that indicates that, the average cross-entropy should be decreasing \emph{for all test images}, because of Theorem \ref{th:exch} and the convexity of the cross-entropy. On the other hand, Figure~\ref{fig:accuracy_derma} illustrated that the accuracy of the model was improving for images well-classified by the asymptotic model $\ybar_\infty $, and getting worse for the other points. This is again consistent with our theory, illustrating the dichotomy of Theorem~\ref{th:class-error}.

One result that was not illustrated in the introduction was the theorem for smooth nonconvex losses (Theorem~\ref{th:smooth-ncvx}). To fill this gap, we look at the evolution of the sigmoid loss (see Figure \ref{fig:sigmoid}) as the dropout ensemble grows. The result, displayed in Figure \ref{fig:sigmo_derma}, tells a story similar to the one of the classification error: the loss averaged over images well-classified by $\ybar_\infty $ is going down, and going up when averaged over wrongly classified images, consistently with Theorem~\ref{th:smooth-ncvx}.

\review{

To also illustrate our results in the multiclass setting, we consider a version of the same dataset with an additional class, corresponding to basal-cell carcinoma, which is the most common type of skin cancer. We present our results in Figure~\ref{fig:multi}. Here, while the assumptions of Theorem \ref{th:class-error} are difficult to check in practice, the conclusion of the theorem seems to hold: the accuracy increases for correct predictions and decreases for completely incorrect predictions. For predictions that are neither, the accuracy is empirically decreasing. We conjecture that some extension of Theorem \ref{th:class-error} should explain this phenomenon.

\begin{figure}
\centering
\includegraphics[width = \columnwidth]{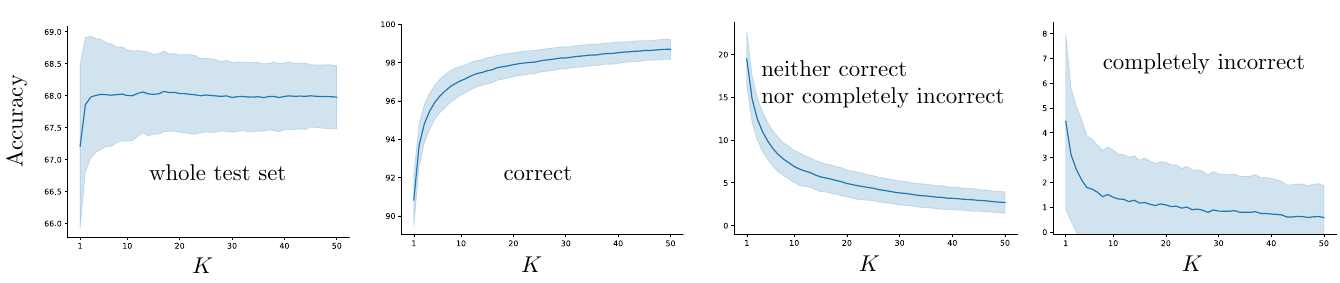}
\caption{\review{\label{fig:multi}Evolution of the accuracy of a dropout ensemble on the 3-class
dermatology data set as the number of models grows (mean and standard deviation over $500$ repetitions). \emph{(Left)} The accuracy averaged over the whole test set has a non-monotonic behaviour.  \emph{(Centre-left)}  The accuracy averaged over images for which the prediction is asymptotically correct is improving.  \emph{(Centre-Right)} The accuracy averaged over images for which the prediction is asymptotically incorrect (yet the correct label is not the one with the smallest predicted probability) is degrading.  \emph{(Right)} The accuracy averaged over images for which the prediction is asymptotically completely incorrect (\emph{i.e.} the correct label has the smallest predicted probability) is degrading. Theorem \ref{th:class-error} provides a theoretical explanation of the behaviour of centre-left and right panels.}}
\end{figure}

}

\subsection{Wisdom of crowds for movie rating predictions}

As an example of prediction of a continuous value, we now turn to an experiment closer to social sciences than machine learning. The goal is to predict the rating of an upcoming movie, using the wisdom of crowds principle, based on data collected by \citet{simoiu2019studying}.

Several people were asked to come up with a prediction for the future ‘‘audience score" rating of an upcoming movie on the website \url{www.rottentomatoes.com}. The full dataset contains $20$ movies, that were rated (out of $100$) several months before their release, by around $450$ users (the minimum number of ratings is $447$ and the maximum $478$).

\begin{figure}
\centering
\includegraphics[width = \columnwidth]{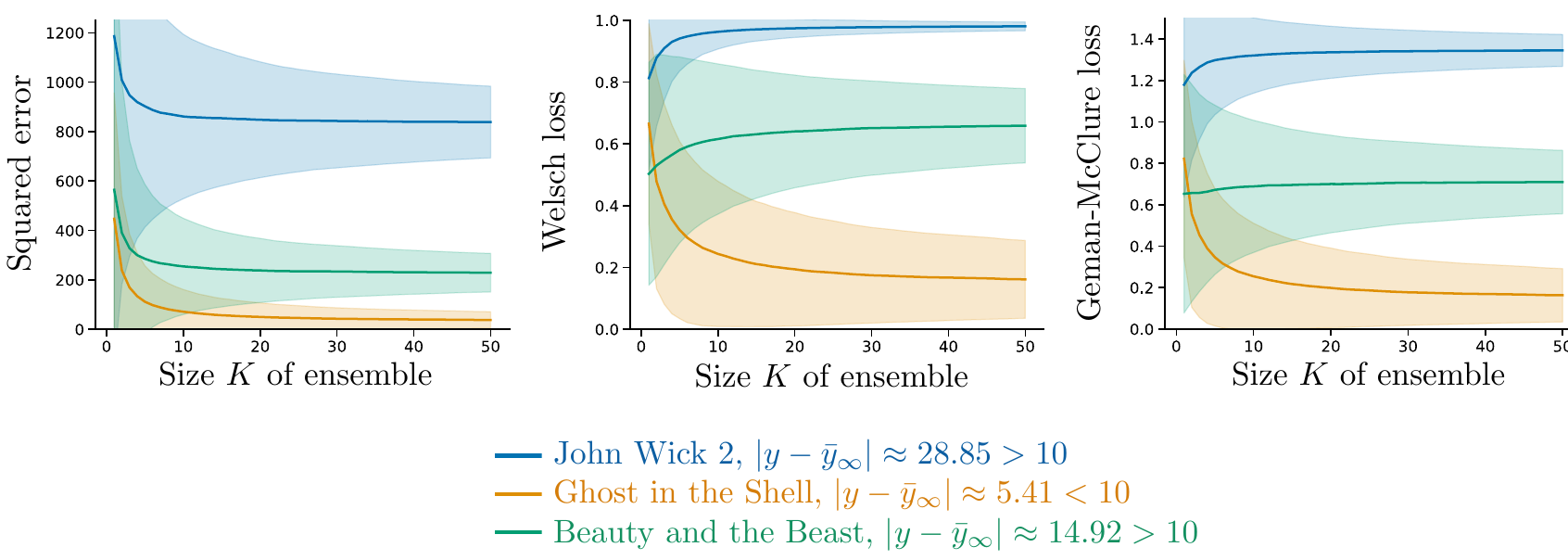}
\caption{\label{fig:crowds1}Evolution of the error at predicting the ratings of three movies as the size of the crowd grows (mean and standard deviation over $10,000$ repetitions). \emph{(Left)} The squared error is decreasing for all three movies. Since this is a convex loss, this is in line with Theorem~\ref{th:exch}. \emph{(Middle and Right)} The Welsh and Geman-McClure losses decreases for Ghost in the Shell, but increases for the other two movies. This is in line with Theorem~\ref{th:smooth-ncvx}, which implies that, for these two losses, the loss should be eventually decreasing when $|y - \bar{y}_\infty| < 10$, and increasing when $|y - \bar{y}_\infty| > 10$.}
\end{figure}

For each movie and for each $K \in \{1,\ldots, 50\}$, we sample randomly (with replacement) $K$ of its ratings $\yhat_1 \ldots, \yhat_K \in [1,100]$ to build an ensemble prediction $\ybar_K\in [1,100]$. We also have a ground truth rating $y \in [1,100]$ for each movie, that was collected on Rotten Tomatoes after it came out.
Our goal is to study if the ensemble prediction of a crowd of $K$ is more accurate that the one of $K+1$. To this end, we look in Figure \ref{fig:crowds1} at three losses for predicting a continuous value: the squared error, the Welsh loss, and the Geman-McClure loss (for more details on these two losses, see Figure \ref{fig:sigmoid}).

The squared error is convex, and consequently decreases, as predicted by Theorem~\ref{th:exch}. 
The two other losses, on the other hand, are non-convex. 
Their hyperparameters were chosen such that both losses are strictly convex when the error is small ($|y - \bar{y}_\infty| < 10$), and strictly concave when $|y - \bar{y}_\infty| > 10$, as in Fig.~\ref{fig:sigmoid}. 
Since these two losses are additionally smooth (\emph{i.e.}, infinitely differentiable, \citealp{barron2019general}), we can use Theorem~\ref{th:smooth-ncvx} to predict that they will go up for John Wick 2 and Beauty and the Beast (because the errors are larger than $10$) and go down for Ghost in the Shell (because the error is smaller than $10$). 

In Appendix~\ref{app:more_movies}, we show similar curves for three more movies. It is also possible to produce such curves for all $20$ movies using a Python notebook available at \url{https://github.com/pamattei/Getting-Better-Ensembles}.

\section{Conclusion}

Our initial question was \emph{‘‘Is it always true that an ensemble of $K+1$ models performs better than an ensemble of $K$ models?''} When the loss is convex and the ordering of the ensemble does not matter, Theorems~\ref{th:exch} and~\ref{th:strong} provide a clear and positive answer to this question. 
When the loss is non-convex, however, a more nuanced picture must be drawn: ensembles will get better in some settings, and worse in others. 
\review{Namely, ensembles will keep getting better (respectively worse) when their asymptotic prediction falls in a locally convex (respectively concave) spot of the loss.

\subsection{Practical guidelines}

 Our work leads us to strengthen the practical guidelines provided by \citet[Section 5.2]{probst2018tune}: \emph{in general, ensembles should be as large as computationally feasible}. Indeed, we showed that this is always true for convex losses. Furthermore, even for nonconvex losses, it is likely that, when averaged over the whole dataset, the loss will be decreasing, since it is reasonable to expect to see more ‘‘good'' points than ‘‘bad'' points.  

An additional empirical guideline is related to evaluation of ensemble techniques. When a new technique is proposed, we recommend to report both convex and nonconvex losses when it is possible. Indeed, since \emph{any} ensemble will get better for a convex loss, showing such improvements will hardly be useful to assess the quality of an ensemble technique. Adding other nonconvex metrics can paint a more detailed picture. Concretely, when evaluating ensembles of classifiers, we recommend to always show, at the very least, the cross-entropy (and/or the Brier score) and the classification error. This guideline is already followed quite often (e.g. by  \citealp{lakshminarayanan2017simple}, Figure 2)

\subsection{Future work}}

Interesting future works include considering other aggregation functions than the mean, or looking at the whole distribution of $L(\ybar_k)$ and not just its mean. In particular, it would be interesting to see if the variance of the loss, called \emph{algorithmic variance} by \citet{lopes2019estimating}, is decreasing or not. 
\review{

All our results were at the individual level, and can be generalised to losses averaged over datasets using the fact that averages of decreasing sequences are decreasing. However, this reasoning will not work for metrics that are not simple averages. For instance, it would be interesting to study the area under the receiver operating characteristic curve or the median squared and errors. \citet{probst2018tune} noticed non-monotonic behaviours for these three losses.

Finally, it would be interesting to assess whether or not ensembling can improve model calibration. This question was investigated by researchers working on scoring rules. In particular, it was shown that ensembles of calibrated models can be uncalibrated \citep{hora2004probability,ranjan2010combining,gneiting2013combining}. It would be interesting to study under which conditions calibration may be monotonically improving.

}

\acks{This work was supported by the French government, through the 3IA Côte d’Azur Investments in the Future project managed by the National Research Agency (ANR) with the reference number ANR-19-P3IA-0002. DG also acknowledges the support of ANR through project NIM-ML (ANR-21-CE23-0005-01) and of EU Horizon 2020 project AI4Media (contract no. 951911).
{Parts of this work were done as DG was employed at Universit\'e C\^ote d'Azur.} The authors thank Raphaël Razafindralambo for spotting an error in an earlier version of this paper.}

\vskip 0.2in
\bibliography{sample}

\newpage 
\input{appendix}

\end{document}

%% file: appendix.tex
\appendix

\section{Proof of strict improvements under strong convexity}
\label{app:strong}

We prove here Theorem \ref{th:strong}. Strong convexity leads to the following strengthening of Jensen's inequality (see, \emph{e.g.}, \citealp{nikodem2014strongly}, Theorem~2).

\begin{theorem}[Jensen improvement]
\label{th:strong_jensen}
If $L: C \rightarrow \mathbb{R}$ is $\mu$-strongly convex, then, for all $x_1,\ldots,x_K$, and
	\begin{equation}
		\label{eq:nikodem}
		L\left(\frac{1}{K}\sum_{k=1}^K   x_k \right) \leq \frac{1}{K}\sum_{k=1}^K  L\left( x_k \right) - \mu  \frac{1}{K}\sum_{k=1}^K \left|\left|  x_k - \frac{1}{K}\sum_{j=1}^K  x_j\right|\right|^2.
	\end{equation}
\end{theorem}
Our goal is to combine this result with our basic identity Eq.~\eqref{eq:triangle} to prove Theorem~\ref{th:strong}.
We use Theorem~\ref{th:strong_jensen} with $x_j = \frac{1}{K-1}  \sum_{k \neq j}  \hat{y}_k$. 
Using Eq.~\eqref{eq:triangle} and the exchangeability assumption, we find that
\begin{equation}
\mathbb{E}\left[	L\left(\frac{1}{K}\sum_{k=1}^K   \hat{y}_k \right) \right] \leq \mathbb{E} \left[ L\left( \frac{1}{K-1} \sum_{k = 1}^{K-1} \hat{y}_k \right)   \right] - \mu \mathbb{E}\left[ \left|\left| \frac{1}{K-1}\sum_{k=1}^{K-1}   \hat{y}_k - \frac{1}{K}\sum_{k=1}^K   \hat{y}_k\right|\right|^2\right].
\end{equation}
We can then compute the expectation on the right-hand side
\small
\begin{align}
\mathbb{E}\left[ \left|\left| \frac{1}{K-1}\sum_{k=1}^{K-1}   \hat{y}_k - \frac{1}{K}\sum_{k=1}^K   \hat{y}_k\right|\right|^2\right] &= \frac{1}{K^2(K-1)^2} 	\mathbb{E}\left[ \left|\left| K\sum_{k=1}^{K-1}   \hat{y}_k - (K-1) \left(\hat{y}_K + \sum_{k=1}^{K-1} \hat{y}_k \right)\right|\right|^2\right] \nonumber \\
&= \frac{1}{K^2(K-1)^2} \mathbb{E}\left[ \left|\left| \sum_{k=1}^{K-1}   \hat{y}_k - (K-1) \hat{y}_K \right|\right|^2\right] \nonumber \\
&= \frac{1}{K^2(K-1)^2} \textup{tr} \left( \textup{Cov}\left( \sum_{k=1}^{K-1}   \hat{y}_k - (K-1) \hat{y}_K \right) \right). \label{eq:strongproof} 
\end{align}
\normalsize
We can now use Bienaymé's identity and exchangeability to compute the trace of the covariance:
\begin{multline}
\textup{tr} \left( \textup{Cov}\left( \sum_{k=1}^{K-1}   \hat{y}_k - (K-1) \hat{y}_K \right) \right) =  (K-1) \textup{tr} \left( \Sigma \right)+ (K-1)(K-2)\textup{tr} \left( \textup{Cov}\left(\yhat_1,\yhat_2\right)\right)   \\ + (K-1)^2 \textup{tr} \left( \Sigma \right)- 2 \textup{tr}\left( \textup{Cov}\left( \sum_{k=1}^{K-1}   \hat{y}_k, (K-1) \hat{y}_K \right)\right),
\end{multline}
finally leading to 
\begin{equation}
\textup{tr} \left( \textup{Cov}\left( \sum_{k=1}^{K-1}   \hat{y}_k - (K-1) \hat{y}_K \right) \right) = K(K-1) \textup{tr} (\Sigma) (1- \rho)
\, .
\end{equation}	
Plugging this expression into Equation \eqref{eq:strongproof} gives the desired bound.
$\blacksquare$

\section{More details on nonconvex losses}

\label{app:spherical}

\subsection{More details on the spherical score}

Consider the setting of Section~\ref{sec:smoothnoncvx}.
The spherical score is defined as 
\begin{equation}
L_1(\yhat) = - \frac{\yhat}{\sqrt{\yhat^2+(1-\yhat)^2}}
\, ,
\end{equation}
when the true label is $y = 1$. When the true label is $0$, as in Section~\ref{sec:smoothnoncvx}, it is defined as
\begin{equation}
L_0(\yhat) = - \frac{1-\yhat}{\sqrt{\yhat^2+(1-\yhat)^2}}.
\end{equation}
The second derivative of $L_0$ is $\yhat \mapsto (-4 \yhat^2 + \yhat + 1)/(2 (\yhat - 1) x + 1)^{5/2}$, whose unique root in $[0,1]$ is $1/8 + \sqrt{17}/8 \approx 0.64$. Therefore, as seen in Figure~\ref{fig:sigmoid}, $L_0$ will be strictly convex on $[0,1/8 + \sqrt{17}/8 )$ and strictly concave on  $(1/8 + \sqrt{17}/8 ,1]$.

\review{
\subsection{More details on the tangent loss}

The tangent loss \citep{masnadi2010design} is a loss function for binary classification.  There are two labels, $y=1$ and $y=-1$, and the output of our model is a score $\hat{y} \in \mathbb{R}$ that indicates its confidence in predicting that the label is $1$. When the true label is $y \in \{-1,1 \}$, it is defined as 
\begin{equation}
    L_y(\hat{y}) =  (2 \textup{arctan}(y \, \hat{y}) - 1)^2.
\end{equation}

This function is convex around its minimum of one, but nonconvex when preditions are overconfident. Indeed, when $y = 1$, the fonction is locally convex on $[y_{-},y_+]$ with $y_- \approx -0.51$ and  $y_+ \approx 1.26$, and locally concave elsewhere.
}

\section{Proof of the theorem for smooth losses (Theorem \ref{th:smooth-ncvx})}
\label{sec:proof_smooth}

We only treat the case $H(\hat{y}_\infty) \succ 0 $, since the case $H(\hat{y}_\infty) \prec 0 $ follows by replacing $L$ by $-L$.
We will base our proof on the following lemma, which generalises a result of \citet[Section 4.3]{small2010expansions}, and is obtained by Taylor-expanding the loss up to the fourth order. In this context, such a technique is often called the delta method for moments (see, \emph{e.g.},  \citealp{small2010expansions,bickel2015mathematical}). 
The proof of the lemma is slightly tedious, and is provided after the main theorem.

\begin{lemma}[\textbf{fourth order delta method}]
\label{lem:delta}
Let $X_1,\ldots,X_n \in \mathbb{R}^d$ be a sequence of nondegenerate i.i.d. random variables, whose first $5$ moments are finite, with mean $\mu$ and covariance matrix $\Sigma \succeq 0$. 
Let $L$ be a function with continuous and bounded partial derivatives of order up to 5, with Hessian matrix $H$. Then, there exists a constant $\alpha \in \mathbb{R}$, that depends on the first $3$ moments of $X$ and the partial derivatives of the loss of order up to $4$, such that, when $K$ goes to infinity
\begin{equation}
\mathbb{E}[L(\bar{X}_K)] = L(\mu) +\frac{\textup{tr}(H(\mu) \Sigma)}{2K}+ \frac{\alpha}{K^2} + \mathcal{O}\left(\frac{1}{K^{5/2}}\right)
\, .
\end{equation}
\end{lemma}

\review{
\noindent
\emph{Proof of Theorem~\ref{th:smooth-ncvx}.}
}
Applying Lemma~\ref{lem:delta} successively to $\bar{y}_K$ and $\bar{y}_{K+1}$ and taking the difference between the two terms, we obtain 
\begin{align}
\mathbb{E}[L(\bar{y}_{K})] - \mathbb{E}[L(\bar{y}_{K+1})] &= \frac{\textup{tr}(H(\hat{y}_\infty) \Sigma)}{2} \left(\frac{1}{K} - \frac{1}{K+1}\right) + \alpha \left(\frac{1}{K^2} - \frac{1}{(K+1)^2}\right) + \mathcal{O}\left(\frac{1}{K^{5/2}}\right) \\
	&= \frac{\textup{tr}(H(\hat{y}_\infty) \Sigma)}{2} \frac{1}{K(K+1)} + \alpha \frac{1+2K}{K^2(K+1)^2}  + \mathcal{O}\left(\frac{1}{K^{5/2}}\right).
\end{align}
Therefore,
\begin{align}
	K(K+1)\left( \mathbb{E}[L(\bar{y}_{K})] - \mathbb{E}[L(\bar{y}_{K+1})] \right)
	&= \frac{\textup{tr}(H(\hat{y}_\infty) \Sigma)}{2} + \alpha \frac{1+2K}{K(K+1)}  + \mathcal{O}\left(\frac{1}{K^{1/2}}\right),
\end{align}
which implies that $K(K+1)\left( \mathbb{E}[L(\bar{y}_{K})] - \mathbb{E}[L(\bar{y}_{K+1})] \right)$ will converge to $\textup{tr}(H(\hat{y}_\infty) \Sigma)/2$ when~$K$ goes to infinity. 
Therefore, for $K$ large enough, the sign of $\mathbb{E}[L(\bar{y}_{K})] - \mathbb{E}[L(\bar{y}_{K+1})]$ is constant and exactly equal to the sign of $\textup{tr}(H(\hat{y}_\infty) \Sigma)$.

Since $H(\hat{y}_\infty) \succ 0$, $\Sigma \succeq 0$, and $\Sigma \neq 0$ because we assumed that the random variable is nondegenerate, a standard property of symmetric matrices (see, \emph{e.g.}, \citealp{serre2010}, Proposition~6.1) ensures that $H(\hat{y}_\infty) \Sigma$ is diagonalisable, with only real nonnegative eigenvalues (at least one of them being positive). This implies that $\textup{tr}(H(\hat{y}_\infty) \Sigma)>0$, and, in turn, that for $K$ large enough, $\mathbb{E}[L(\bar{y}_{K+1})] < \mathbb{E}[L(\bar{y}_{K})]$. 
\qed

\subsection{Proof of the fourth order delta method (Lemma~\ref{lem:delta})}

Lemma \ref{lem:delta} was proven in the univariate case by \citet[Section 4.3]{small2010expansions} and we extend it here to the multivariate case. We will see it as a consequence of the multivariate version of delta method presented by \citet{bickel2015mathematical}.

Under the assumptions of Lemma~\ref{lem:delta}, Theorem~5.3.2 of \citet{bickel2015mathematical} implies
\begin{multline}
	\label{eq:bickel}
	\mathbb{E}[L(\bar{X}_K)] = L(\mu) +  \sum_{r=2}^{4} \frac{1}{r!} \sum_{j_1,j_2,\ldots ,j_r = 1}^d \frac{\partial^r L}{\partial x_{j_1} \cdots \partial x_{j_r}}(\mu)\mathbb{E}[ (\bar{X}_{Kj_1} - \mu_{j_1})\cdots(\bar{X}_{Kj_r} - \mu_{j_r})] \\ + \mathcal{O}(K^{-5/2})
 \, .
\end{multline}
Let us look at the quadratic term. We have
\begin{align}
	\frac{1}{2} \sum_{j_1,j_2= 1}^d \frac{\partial^2 L}{\partial x_{j_1} \partial x_{j_2}}(\mu)\mathbb{E}[ (\bar{X}_{Kj_1} - \mu_{j_1})(\bar{X}_{Kj_2} - \mu_{j_2})] &=\frac{1}{2} \mathbb{E} \left[ (\bar{X}_{K} - \mu) ^T H(\mu) (\bar{X}_{K} - \mu)\right] \\
	&= \frac{1}{2} \textup{tr}(H(\mu) \textup{Cov}(\bar{X}_{K})) = \frac{\textup{tr}(H(\mu) \Sigma)}{2K}.
\end{align}

Let us now look at the third and fourth order terms. We need to assess how $K$ affects the multivariate moments of the empirical mean. Let us define first the multivariate moments of $X$ as, for $r \in \{3,4\}$ and $j_1,\ldots,j_r \in \{1,\ldots,d\}$
\begin{equation}
\mu_{j_1,\ldots,j_r} = \mathbb{E}[ (X_{j_1} - \mu_{j_1})\ldots (X_{j_r} - \mu_{j_r})]
\, ,
\end{equation}
The moments of $\bar{X}_K$ can now be expressed in terms of moments of $X$. Following for instance \citet[p.~54]{hall2013bootstrap} or \citet[Section~2.3]{mccullagh2018tensor}, we have 
\begin{equation}
\mathbb{E}[ (\bar{X}_{Kj_1} - \mu_{j_1})(\bar{X}_{Kj_2} - \mu_{j_2})(\bar{X}_{Kj_3} - \mu_{j_3})] = \frac{1}{K^2} \mu_{j_1,j_2,j_3}
 \, ,
\end{equation}
and 
\begin{multline}
	\mathbb{E}[ (\bar{X}_{Kj_1} - \mu_{j_1})(\bar{X}_{Kj_2} - \mu_{j_2})(\bar{X}_{Kj_3} - \mu_{j_3})(\bar{X}_{Kj_4} - \mu_{j_4})] = \\ \frac{\mu_{j_1,j_2}\mu_{j_3,j_4} + \mu_{j_1,j_3}\mu_{j_2,j_4} + \mu_{j_1,j_4}\mu_{j_2,j_3}}{K^2}  + \mathcal{O}\left( \frac{1}{K^3}\right),
\end{multline}
where the constant in the $\mathcal{O}$ is a multivariate cumulant.
Plugging these expressions into Equation \eqref{eq:bickel} finally gives
\begin{equation}
	\mathbb{E}[L(\bar{X}_K)] = L(\mu) +\frac{\textup{tr}(H(\mu) \Sigma)}{2K}+ \frac{\alpha}{K^2} + \mathcal{O}\left(\frac{1}{K^{5/2}}\right),
\end{equation}
with
\begin{multline}
	\alpha  = \frac{1}{3!} \sum_{j_1,j_2,j_3 = 1}^d \frac{\partial^3 L}{\partial x_{j_1}  \partial x_{j_2}\partial x_{j_3}}(\mu) \mu_{j_1,j_2,j_3} \\ + \frac{1}{4!} \sum_{j_1,j_2,j_3,j_4 = 1}^d  \frac{\partial^4 L}{\partial x_{j_1}  \partial x_{j_2}\partial x_{j_3}\partial x_{j_4}}(\mu)\left ( \mu_{j_1,j_2}\mu_{j_3,j_4}  + \mu_{j_1,j_3}\mu_{j_2,j_4} + \mu_{j_1,j_4}\mu_{j_2,j_3} \right).
\end{multline}
One can check that, in the univariate case ($d=1$), our expression for $\alpha$ is identical to the one given by \citet[Section 4.3]{small2010expansions}. \qed

\section{Omitted proofs for Section~\ref{sec:monotonicity-tail-probabilities}}
\label{sec:proof-monotonicity-tail-proba}

\subsection{Proof of Theorem~\ref{th:monotonicity-tail-proba}}
\label{sec:proof-monotonicity-tail-proba-univariate}

For all $i \geq 1$, let us set $Y_i\defeq X_i - \mu$. 
Note that, by definition, the $Y_i$s are i.i.d. and centered. 
We are going to apply two theorems of \citet{petrov1965probabilities} to the random variables $Y_i$s.
Let us recall briefly Petrov's results here. 
We start with some notation: 
$V(x)$ denotes the cumulative distribution function of $X_1$, 
$\mathfrak{U}$ denotes the points of increase of $V$ and $A$ its supremum, 
$\mathcal{B}^+$ denotes the set of non-negative values $h$ such that $\int_0^{+\infty}\exps{hx}\Diff V(x) < +\infty$ and $B$ its supremum. 
Further, assuming $B>0$, for any $0<h<B$, we define 
\[
R(h) = \int_{-\infty}^{+\infty} \exps{hx} \Diff V(x)
\, , \qquad 
m(h) = \frac{1}{R(h)}\int_{-\infty}^{+\infty} x\exps{hx} \Diff V(x)
\qquad \text{and}\qquad
\sigma(h)^2 = \frac{\Diff m(h)}{\Diff h}
\, .
\]
Finally, define $A_0=\lim_{h\to B^-} m(h)$. 
We are now able to state the results:

\begin{theorem}[Theorem~3 and Theorem~6 from \citet{petrov1965probabilities}]
\label{th:petrov}
Let $Y_1,Y_2,\ldots$ be a sequence of i.i.d. random variables. 
\begin{itemize}
\item Assume that the $Y_i$s are non-lattice such that $B>0$. 
Let $c$ be any constant satisfying the condition $\expec{Y_1} < c < A_0$. 
Then 
\[
\proba{Y_1+ \cdots + Y_n \geq nc}  = \frac{\exp\left(n\log R(h) - nhc\right)}{h\sigma(h)\sqrt{2\pi n}}(1+\littleo{1})
\, .
\]
Here $h$ is the unique real root of the equation $m(h)=c$. 
\item Assume that the $Y_i$ are lattice distributed with values in $a+NH$, $N\in\Relint$. 
Assume that $\expec{Y_1}> -\infty$, $A_0 < +\infty$, and $B>0$. 
If $nx\in na+\Relint H$, then 
\[
\proba{Y_1+\cdots + Y_n = nx} = \frac{H\exp \left( n\log R(h) - nhx\right)}{\sigma(h)\sqrt{2\pi n}} \left( 1 + \mathcal{O}\left(\frac{1}{n}\right) \right)
\, ,
\]
and 
\[
\proba{Y_1 + \cdots + Y_n \geq nx} \leq \frac{H \exp\left(n\log R(h) - nhx\right) }{\sigma(h)\sqrt{2\pi n}(1-\exps{-Hh})}\left(1+\mathcal{O}\left(\frac{1}{n}\right)\right)
\, 
\]
as $n\to +\infty$ uniformly in $x\in [\expec{Y_1}+\epsilon, A_0-\epsilon]$, with $\epsilon > 0$ arbitrary. 
Here $h$ is the unique real root of the equation $m(h)=x$. 
\end{itemize}
\end{theorem}

Theorem~\ref{th:petrov} gives us a control of the tail probability, which is precise enough to show monotonicity. 
There are two cases to consider: the \emph{non-lattice} case (assuming (1), (2), and (3)) and the \emph{lattice} case (assuming (1), (2), and (3bis)). 
We denote $p_n  = \proba{\Xbar_n \geq \mu + \varepsilon}$, for $\varepsilon$ satisfying the assumptions of the theorem.

\paragraph{Non-lattice case.}
Under assumption (3), we know that the $Y_i$s are non-lattice random variables. 
Let us check carefully the assumptions of Petrov's result: 
Under assumption (1), we know that $B=+\infty$. 
Under assumption (2), we know that $\proba{Y_1 > \varepsilon} >0$. 
In the language of \citet{petrov1965probabilities}, it implies that there is a growth point of $Y_1$ greater than $\varepsilon$. 
Thus $A$
is greater than $\varepsilon$. 
Since we are in the case $B=+\infty$, we know that $A_0=A$, and in particular $A_0 > \varepsilon > 0$.
Thus we can apply the first part of Theorem~\ref{th:petrov} to the $Y_i$s with $c=\varepsilon$. 
We obtain 
\begin{equation}
\label{eq:pn-non-lattice}
\forall n\in\Integers^\star, \qquad p_n = \frac{\expl{n\log R(h) - nh\varepsilon}}{h\sigma(h)\sqrt{2\pi n}}(1+u_n)
\, ,
\end{equation}
where $u_n =\littleo{1}$ as $n\to +\infty$. 
From Eq.~\eqref{eq:pn-non-lattice}, we deduce that, for all $n\geq 1$,
\begin{equation} 
\label{eq:pn-non-lattice-aux-1}
\frac{p_{n+1}}{p_n} = \rho \sqrt{\frac{n}{n+1}}\frac{1+u_{n+1}}{1+u_{n}}
\, ,
\end{equation}
where $\rho \defeq \expl{\log R(h) - h\varepsilon}$. 
By the first lemma in \citet{petrov1965probabilities}, $m$ is strictly increasing and continuous in $(0,+\infty)$. 
Thus, by definition of $h$, the mapping $\alpha: t\mapsto \log R(t) - t\varepsilon$ is decreasing for $t<h$ and increasing for $t >h$, attaining a (strict) minimum in $t=h$. 
Since $\alpha(0)=0$, we deduce that $\alpha(h)<0$, which in turn implies $\rho <1$ in Eq.~\eqref{eq:pn-non-lattice-aux-1}. 
Since both $\sqrt{n/(n+1)}$ and $(1+u_{n+1})/(1+u_n)$ are arbitrarily close to $1$ for large $n$, we deduce the result in the case of a non-lattice distribution. 

\paragraph{Lattice case.}
As in the previous paragraph, under assumption (1) and (2), $B=+\infty$ and $A_0 >\varepsilon >0$. 
Moreover, assumption (3bis) guarantees that $n\varepsilon$ takes values in the lattice. 
Similarly to Eq.~\eqref{eq:pn-non-lattice}, we obtain 
\begin{equation}
\label{eq:pn-lattice}
\forall n\in\Integers^\star, \qquad p_n = \frac{H\expl{n\log R(h) - nh\varepsilon}}{\sigma(h)\sqrt{2\pi n}(1-\exps{-Hh})}(1+u_n)
\, ,
\end{equation}
where 
$u_n=\littleo{1}$. 
From Eq.~\eqref{eq:pn-lattice}, we deduce that, for all $n\geq 1$,
\begin{equation} 
\label{eq:pn-lattice-aux-1}
\frac{p_{n+1}}{p_n} = \rho \sqrt{\frac{n}{n+1}}\frac{1+u_{n+1}}{1+u_{n}}
\, ,
\end{equation}
where $\rho$ is as before. 
We conclude in the same fashion as in the first part of the proof and obtain monotonicity for  $\proba{\Xbar_n \geq \mu + \varepsilon}$. 

Let us now deal with the strict inequality case. 
Decomposing $\proba{\Xbar_n > \mu + \varepsilon}$ as the difference between $\proba{\Xbar_n \geq \mu + \varepsilon}$ and $\proba{\Xbar_n = \mu + \varepsilon}$, and using again Theorem~\ref{th:petrov}, we obtain
\begin{equation}
\label{eq:pn-lattice2}
\proba{\Xbar_n > \mu + \varepsilon}= \frac{H\expl{n\log R(h) - nh\varepsilon}}{h\sigma(h)\sqrt{2\pi n}}\left(\frac{1}{1-\exps{-Hh}} - 1 \right)(1+u_n)
\, ,
\end{equation}
where $u_n=\littleo{1}$. Since $1/(1-\exps{-Hh}) - 1 > 0$, the previous reasoning applies, and $\proba{\Xbar_n > \mu + \varepsilon}$ is also eventually monotonic.
\qed

\begin{remark}
It is possible to relax Assumption (1): a careful reading of \citet{petrov1965probabilities} reveals that one only needs a non-degenerate interval $[0,B)$ on which the moment generating function of $X_1$ is well-defined, with $B>0$. 
But in that event, $\varepsilon$ may need to be taken ``small enough.''
If Assumption (2) is not satisfied, then simple counter-examples such as a Dirac at $\mu$ break the statement. 
Assumption (3bis) cannot be relaxed, as will be seen in the second example of the following discussion.
\end{remark}

\subsection{Proof of Theorem~\ref{th:monotonicity-tail-proba-multivariate}}
\label{sec:proof-monotonicity-tail-proba-multivariate}

Let us set $a\defeq \mu + \epsilon$. 
We first notice that the assumptions of Lemma~\ref{lemma:tau-existence} (which is stated and proved in Section~\ref{sec:existence-tau}) are satisfied, thus there exists $\tau\in U_\alpha$ such that $\nabla \varphi(\tau) = a$ and $\tau_{(1)}>0$. 
We now prove a stronger result than Theorem~\ref{th:monotonicity-tail-proba-multivariate}, namely, for $n$ large enough,
\begin{equation}
\label{eq:xbarn-multivariate-asymptotic}
\proba{\Xbar_n \geq a} = \frac{\expl{-n (\tau^\top a - \varphi(\tau))}}{(2\pi n)^{D/2}\prod_{k=1}^D \tau_k \deter{\nabla^2\varphi(\tau)}^{1/2}}(1+\littleo{1})
\, .
\end{equation}
The result follows from this last display, since $\tau^\top a - \varphi(a) >0$. 
Indeed, $t\mapsto t^\top a - \varphi(t)$ is a strictly concave function, with maximum achieved at $t=\tau$ and value $0$ at $0$. 
All that is left to do is to prove Eq.~\eqref{eq:xbarn-multivariate-asymptotic}. 
We use Theorem~1 in \citet{joutard_2017}. 
More precisely, we apply Theorem~1 to $Z_n=\Xbar_n$, with rate $b_n=n$ and $H=0$. 
We have several assumptions to check. 

First, in \citet{joutard_2017}'s notation, the moment-generating function of $b_nZ_n$ is given by $\phi_n(t)=\phi(t)^n$, since the $X_i$s are i.i.d. 
Simply by taking the log, we obtain that the normalized cumulant generating function of $b_nZ_n$ is given by $\varphi_n(t)=\log \phi(t)$. 
By assumption, $\phi$ is defined and differentiable on $U_\alpha$.  
Since $\phi$ is positive by construction, we have shown the existence of a differentiable function $\varphi$ such that $\lim_n \varphi_n(t)=\varphi(t)$ for all $t\in U_\alpha$. 

We now proceed to show that assumption A1-3 are satisfied. 
We start with A1, asking that $\varphi_n$ is an holomorphic function on $D_\alpha^p$, where $D_\alpha\defeq \{z\in\Complex \}$. 
We notice that, since $\varphi_n=\varphi$, we only have to show the claim for $\varphi$. 
Let us fix all variables except the $i$th. 
That is, let us consider the mapping $\varphi_i : s \mapsto \varphi(t_1,\ldots,t_{i-1},s,t_{i+1},\ldots,t_p)$ for some $t_1,\ldots,t_{i-1},t_{i+1},\ldots,t_p\in\Reals$. 
By our assumption on $\varphi$, $\varphi_i$ is well-defined on the segment $(-\alpha,\alpha)$, and therefore $\varphi_i$ as a function of the complex variable is well-defined and holomorphic on the strip $\{z=a+\ii b, a\in (-\alpha,\alpha), b\in\Reals\}$. 
In particular, $\varphi_i$ is holomorphic on the disk $D_\alpha$. %
Hartogs' theorem tells us that, since $\varphi:D_\alpha^p\to \Complex$ is a continuous function. 
Finally, Osgood's Lemma guarantees that $\varphi$ is holomorphic on $D_\alpha^p$. 
That it is bounded is straightforward from our assumption. 

Let us now turn to A2. 
Again, since $\varphi_n=\varphi$, the first part of the assumption is trivial to check (one can take $H(t)=0$). 
In order to prove that $\nabla^2\varphi(\tau)$ is positive definite, we apply exponential tilting, also called by  \citet{petrov1965probabilities} the \emph{Cram\'er method}. 
More precisely, let us define the random vector $Z$ with density with respect to Lebesgue measure on $\Reals^p$ defined as
\[
\frac{1}{\phi(\tau)}\expl{\inner{\tau}{x}}\rho(x)
\, ,
\]
where, as we recall, $\rho$ is the density of $X$. 
Then one can readily check that $\nabla^2\varphi(\tau)=\cov{Z}$. 
In particular, $\nabla^2\varphi(\tau)$ is positive semi-definite. 
Since, by assumption, $X$ has non-zero density everywhere, the same is true for $Z$ by construction. 
In particular, $Z$ is not supported in an hyperplane. 
Thus its covariance matrix has no zero eigenvalue and we can conclude.

All that is left to do is to prove that A3 is satisfied. 
We first notice that, according to the discussion at the beginning of this proof, for any $t\in\Reals$, 
\begin{equation}
\label{eq:ratio-A3}
\abs{\frac{\phi_n(\tau + \ii t)}{\phi_n(\tau)}} = \abs{\frac{\phi(\tau + \ii t)}{\phi(\tau)}}
\, .
\end{equation}
Thus showing that $\abs{\frac{\phi(\tau + \ii t)}{\phi(\tau)}}<1$ will guarantee that Eq.~\eqref{eq:ratio-A3} is $\littleo{1/b_n^{p/2}}$. 
But one can recognize that 
\[
\frac{\phi(\tau + \ii t)}{\phi(\tau)} = \expec{\expl{\ii\inner{t}{Z}}}
\, ,
\]
which is the characteristic function of the random vector $Z$ defined in the previous paragraph. 
Again, since $X$ has non-zero density everywhere, the same is true for $Z$ and Lemma~\ref{lemma:feller-multivariate} (which is proved in Section~\ref{sec:feller-multivariate}) guarantees that  $\abs{\expec{\expl{\ii\inner{t}{X}}}} < 1$ for any $t\neq 0$. 
\qed

\section{When monotonicity does not hold}
\label{sec:limitations-monotonicity}

In this section, we give two examples in which the assumptions of Theorem~\ref{th:monotonicity-tail-proba} are not satisfied and $(p_n)_{n\geq 1}$ is not monotonous.
We set $S_n\defeq X_1 + \cdots + X_n$. 

\begin{example}
Let us assume that $X_1$ has a $\alpha$-stable distribution, with density $f(x;\alpha,\beta,c,\mu)$. 
In that case, $S_n-n\mu = \sum_{i=1}^{n}1\cdot (X_i-\mu)$ has density 
\[
\frac{1}{n^{1/\alpha}}f(y/n^{1/\alpha};\alpha,\beta,c,\mu)
\, .
\]
We can directly compute $p_n$ by a change of variables, and we obtain
\begin{align*}
p_n &= \proba{S_n-n\mu > n\epsilon} \\
&= \int_{n\epsilon}^{+\infty} \frac{1}{n^{1/\alpha}} f(y/n^{1/\alpha};\alpha,\beta,c,\mu)\Diff y \\
p_n &= \int_{n^{1-1/\alpha}\epsilon}^{+\infty} f(x;\alpha,\beta,c,\mu)\Diff x
\, .
\end{align*}
Whenever $\alpha < 1$, $\left(n^{1-1/\alpha}\epsilon\right)_n$ is a decreasing sequence, and according to the last display  $p_n$ is \emph{increasing}. 
This is even though the distribution is symmetric around zero. 
We showcase this effect in Figure~\ref{fig:alpha-stable}. 
\end{example} 

\begin{example}
Let us assume that $X_1\sim \bernoulli{p}$, and therefore $S_n\sim \binomial{n,p}$. 
Thus 
\begin{align*}
p_n &= \proba{X_1+\cdots +  X_n \geq n(p+\varepsilon)} \\
&= \sum_{k = \ceil{n(p+\varepsilon)} }^n \binom{n}{k}p^k(1-p)^{n-k}
\, .
\end{align*}
While innocent-looking, this last display has a complicated behavior. 
In particular, it is not monotonic (see Figure~\ref{fig:bernoulli}).
Transforming slightly this example in order to satisfy the assumptions of Theorem~\ref{th:monotonicity-tail-proba} radically changes the behavior of $p_n$, as predicted. 
One possible way to do so is to add a small mass $\varepsilon/(2(\mu+\varepsilon)-1)$ at $\mu+\varepsilon$, while setting the probability to be equal to $0$ and $1$ equal to $(2\mu-1)/(2(\mu+\varepsilon)-1)$. 
We refer to Figure~\ref{fig:bernoulli} for an illustration. 
\end{example}

\begin{figure}
\centering
\includegraphics[scale=0.3]{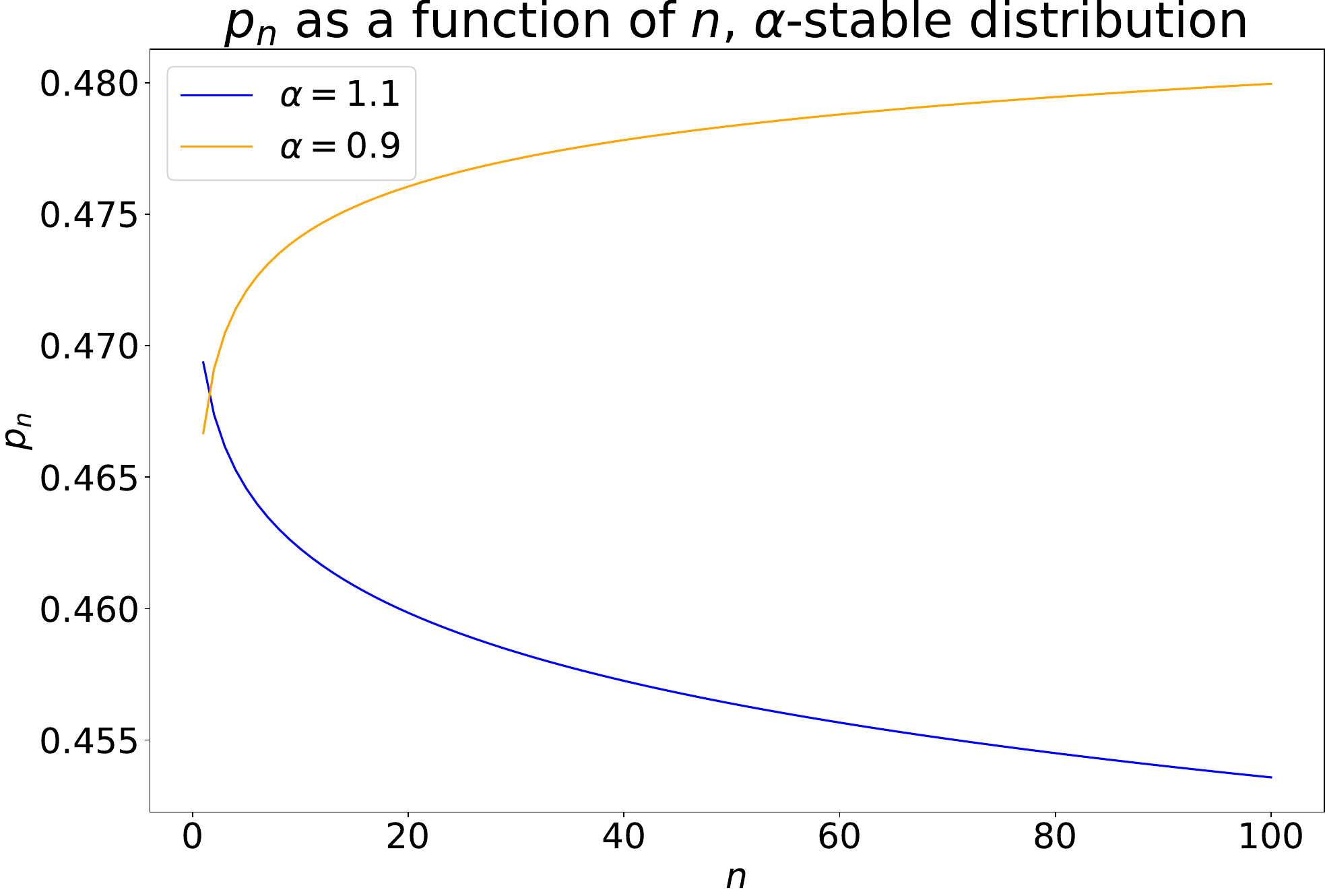}
\caption{\label{fig:alpha-stable}Evolution of $p_n$ when $X_1$ follows the $\alpha$-stable distribution. Here we considered $\epsilon=0.1$. When $\alpha < 1$, $p_n$ is increasing while it is decreasing, as expected, when $\alpha > 1$.}
\end{figure}

\begin{figure}
\centering
\includegraphics[scale=0.3]{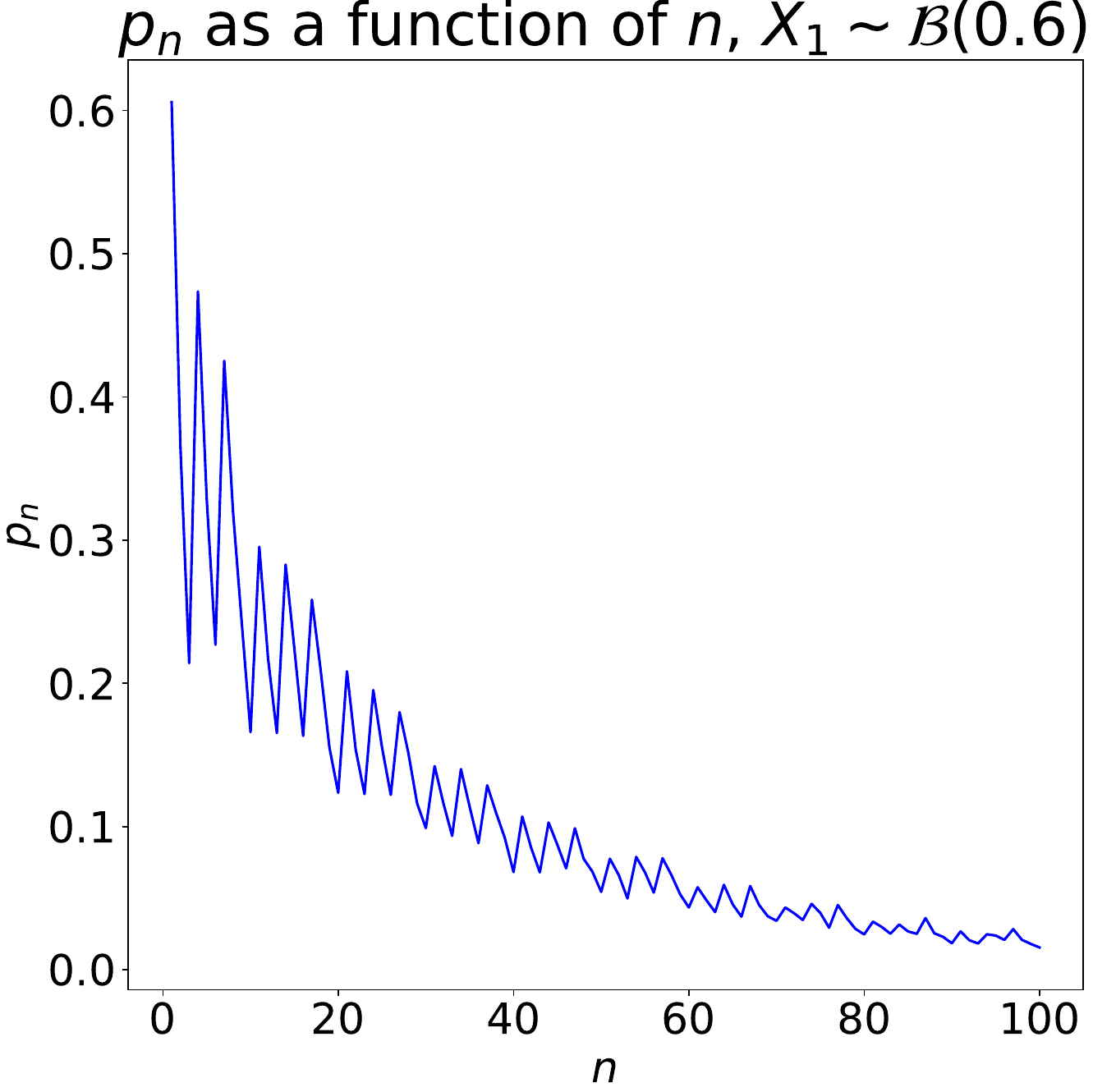}
\hfill 
\includegraphics[scale=0.3]{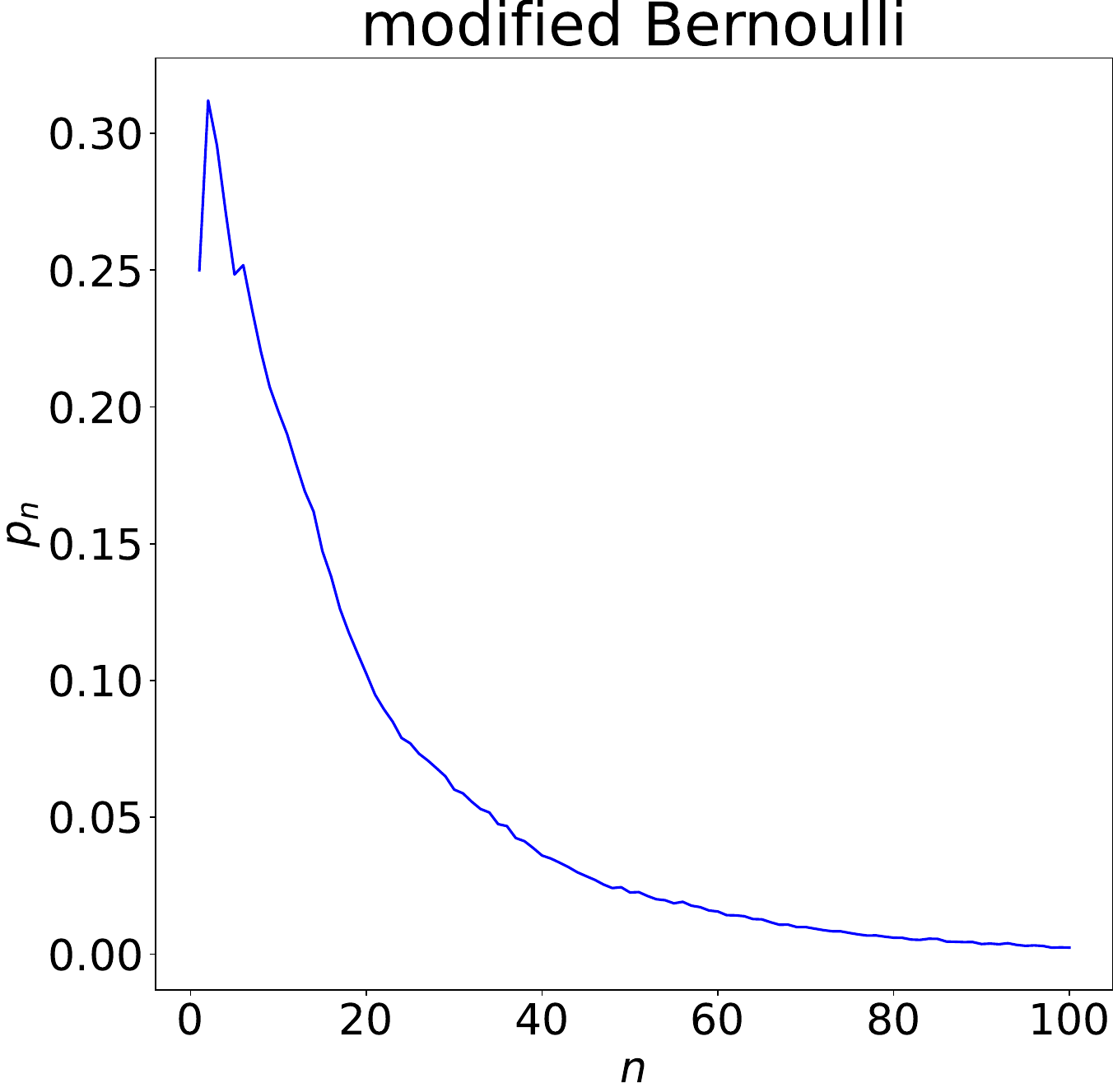}
\caption{\label{fig:bernoulli}\emph{Left:} evolution of $p_n$ when $X_1$ follows a Bernoulli distribution. Here we took $\varepsilon=0.1$. \emph{Right:} same example, with some mass added at $\mu+\varepsilon$. We observe that the sequence $(p_n)_{n\geq 1}$ is decreasing for sufficiently large $n$, as predicted by Theorem~\ref{th:monotonicity-tail-proba}. }
\end{figure}

\section{Technical results}
\label{sec:technical-results}

\subsection{On the existence of \texorpdfstring{$\tau$}{T}}
\label{sec:existence-tau}

For a subset $A$ of $\mathbb{R}^D$, we denote by $C(A)$ its convex hull, and by $\textup{int}(A)$ its interior.

\begin{lemma}[Existence of $\tau$]
\label{lemma:tau-existence}
Let $X \in \mathbb{R}^D$ be a random variable with logaritmic moment generating function~$\varphi$, support $\mathcal{S}$, and mean $\mu$. 
We assume that $X$ is absolutely continuous with respect to the Lebesgue measure, and that there exists $\alpha \in ]0, \infty]$ such that $\{ t \in \mathbb{R}^D | \varphi(t) < \infty \} = U_\alpha$, where $U_\alpha$ denotes the ball of radius $\alpha$. 
Let $\lambdamin$ (resp. $\lambdamax$) denote the minimal (resp. maximal) eigenvalue of $\nabla^2\varphi$. 
Then, for all $\varepsilon \in \mathbb{R}^D$ with strictly positive entries such that $\mu + \varepsilon \in \textup{int}(C(\mathcal{S}))$ and
\begin{equation} 
\label{eq:condition-epsilon-appendix}
\min_{1\leq i\leq D} \frac{\epsilon_i}{\norm{\epsilon}} > \sqrt{2\left(1-\sqrt{\frac{\lambdamin}{\lambdamax}} \right)}
\, ,
\end{equation}
there exists $\tau \in U_\alpha$ with strictly positive entries such that $\nabla \varphi (\tau) = \mu + \varepsilon$.
\end{lemma}

\begin{remark}
Assumption~\eqref{eq:condition-epsilon-appendix} is quite restrictive. 
Indeed, assuming for a moment that $\norm{\epsilon}=1$, we see that the locus of points such that $\epsilon_i>c$ for all $i$ is empty whenever $c\geq 1/\sqrt{d}$. 
Therefore, one needs 
\[
1-\frac{1}{2d}\leq \frac{\lambdamin}{\lambdamax} \left( \leq 1\right)
\]
for Eq.~\eqref{eq:condition-epsilon-appendix} to hold for a non-empty set of points. 
For large $d$, this means requiring nearly isotropic $\nabla^2\varphi$. 
\end{remark}

\begin{proof} 
Let $\varepsilon \in \mathbb{R}^D$ satisfying our assumptions.
Since $\{ t \in \mathbb{R}^D | \varphi(t) < \infty \} = U_\alpha$, $\varphi$ is ``steep'' in the sense of \citet{barndorff_1978}, and his Theorem~9.2 implies then that $ \nabla \varphi (U_\alpha) = \textup{int}(C(\mathcal{S}))$. 
Therefore, there exists $\tau \in U_\alpha$ such that $\nabla \varphi (\tau) = \mu + \varepsilon$.
    
All that is left to do is to show that all entries of $\tau$ are strictly positive. 
Since $\varphi$ is $\lambdamax$-smooth, one can write
\begin{equation}
\label{eq:l-smooth-bound}
\frac{1}{\lambdamax}\norm{\nabla\varphi (\tau) - \nabla \varphi(0)} ^2 \leq (\nabla\varphi(\tau)-\nabla\varphi(0))^\top (\tau - 0) = (\mu + \epsilon - \mu)^\top \tau = \epsilon^\top \tau
\, .
\end{equation}
Thus we have proved that $\epsilon^\top \tau \geq \frac{1}{\lambdamax}\norm{\epsilon}^2$. 
In particular, $\tau$ points in the same direction as $\epsilon$. 
Though this is not sufficient to conclude: the halfspace of directing vector $\epsilon$ contains points having non-positive coordinates. 
We note that, since $X$ is absolutely continuous, $\lambdamin >0$, and $\varphi$ is $\lambdamin$-strongly convex.  
Therefore, 
\begin{equation}
\label{eq:mu-strongly-convex-bound}
(\nabla\varphi(\tau)-\nabla\varphi(0))^\top (\tau - 0) \geq \lambdamin \norm{\tau - 0}^2
\, ,
\end{equation}
that is, $\epsilon^\top \tau \geq \lambdamin \norm{\tau}^2$. 
Combining both bounds, we see that 
\[
\epsilon^\top\tau \geq \sqrt{\frac{1}{\lambdamax}\norm{\epsilon}^2 \cdot \lambdamin \norm{\tau}^2}
\, ,
\]
from which we deduce that 
\begin{equation} 
\label{eq:def-cone}
\cos(\epsilon,\tau) \geq \sqrt{\frac{\lambdamin}{\lambdamax}}
\, .
\end{equation}
Equivalently, the angle between $\epsilon$ and $\tau$ is less than $\theta_0\defeq \arccos\left(\sqrt{\frac{\lambdamin}{\lambdamax}}\right)$, meaning that $\tau$ belongs to $\mathcal{C}$ the half-cone of direction $\epsilon$ and aperture $2\theta_0$. 

Let us now prove that, under our assumption, no point of $\mathcal{C}$ has a negative coordinate. 
Without loss of generality, we can assume that $\norm{\epsilon}=\norm{\tau}=1$. 
Using Eq.~\eqref{eq:def-cone}, we see that 
\[
\norm{\epsilon-\tau}^2 = 1+1-2\epsilon^\top\tau \leq 2\left(1-\sqrt{\frac{\lambdamin}{\lambdamax}} \right) \eqdef \rho^2 
\, .
\]
To put it plainly, $\tau$ belongs to $\mathcal{B}$, the closed ball of center $\epsilon$ and radius $\rho$. 
The distance from the center of $\mathcal{B}$ to each hyperplane $\{x_i=0\}$ is given by $\epsilon_i$. 
Our assumption guarantees that $\epsilon_i > \rho$ for all $i\in [d]$, thus $\mathcal{B}$ does not intersect any of the hyperplanes $\{x_i=0\}$. 
In particular, all point inside $\mathcal{B}$ have positive coordinates, and since $\tau\in\mathcal{B}$ we can conclude. 
\end{proof}

\begin{figure}
    \centering
\begin{tikzpicture}
\draw[thick,->] (0,0) -- (6,0) node[anchor=north west] {$x_i$};
\draw[thick,->] (0,0) -- (0,6);
\draw[->] (0,0) -- (5,2) node[anchor=south] {$\epsilon$};
\draw (0,0) -- (6,0.66*6);
\draw (0,0) -- (6.8,0.18*6.8);
\draw[->] (0,0) -- (6.5,1.8) node[anchor=north] {$\tau$};
\draw[blue] (0,0)+(0:2.5cm) arc (0:21.77:2.5cm) node[anchor=north west] {$\theta$};
\draw[red] (0,0)+(21.77-11.46:3.5cm) arc (21.77-11.46:21.77:3.5cm) node[anchor=south] {$\theta_0$};
\draw[dotted] (5,2) -- (5,0) node[anchor=north] {$\epsilon_i$};
\end{tikzpicture}
    \caption{\label{fig:cone-explanation}Visual help for the proof of Lemma~\ref{lemma:tau-existence}. The vector $\tau$ is contained in a half cone of axis $\epsilon$ and angle $\theta_0$ (in red). 
    }
\end{figure}
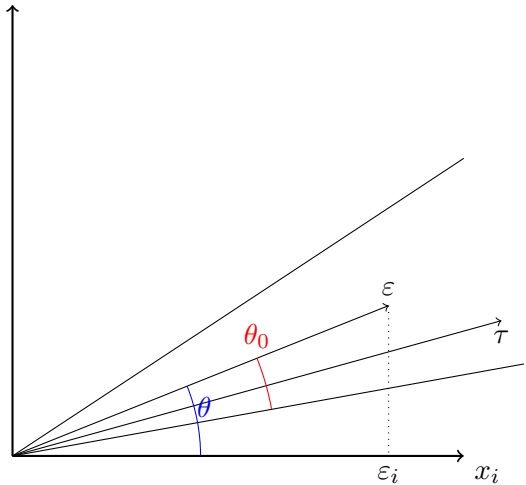

\subsection{On the characteristic function of a random vector with non-zero density}
\label{sec:feller-multivariate}

In this section, we state and prove the following result, which is a straightforward multivariate extension of Lemma~3 in Chapter XV.7 of \citet{feller_1971}. 

\begin{lemma}[Characteristic function of random vector]
\label{lemma:feller-multivariate}
Let $X\in\Reals^p$ be a random vector with non-zero density almost everywhere. 
Then 
\[
\forall t \neq 0, \qquad \abs{\expec{\expl{\ii\inner{t}{X}}}} < 1
\, .
\]
\end{lemma}

\begin{proof}
Let us notice that the triangle inequality guarantees that $\abs{\expec{\expl{\ii\inner{t}{X}}}} \leq 1$ for any $t\in\Reals^p$. 
Therefore, one only needs to show that $\abs{\expec{\expl{\ii\inner{t}{X}}}} \neq 1$, which we will achieve reasoning by contradiction. 

First, let us suppose that there exists $t\neq 0$ such that $\expec{\expl{\ii\inner{t}{X}}} = 1$. 
Then the expectation of the non-negative function $x\mapsto 1-\cos \inner{t}{x}$ is exactly $0$, meaning that $\inner{t}{x} = 0$ modulo $2\pi$ wherever $X$ has non-zero mass. 
Since $t\neq 0$, this implies that $X$ is supported in the collection of hyperplanes defined by 
\[
\inner{t}{x} = 2k\pi, \qquad k\in\Relint
\, , 
\]
which contradicts $X$ having non-zero density everywhere. 

Let us now come back to the original problem, and assume that $\abs{\expec{\expl{\ii\inner{t}{X}}}} = 1$. 
Therefore there exists $b\in\Reals$ such that $\abs{\expec{\expl{\ii\inner{t}{X}}}}=\exps{\ii b}$. 
Since $t\neq 0$, there exists a vector $c\in\Reals^p$ such that $b=\inner{t}{c}$, and one can write
\[
\expec{\expl{\ii \inner{t}{X - c}}} = 1
\, .
\]
We recognize the characteristic function of the random vector $Y\defeq X-c$, which also has non-zero density everywhere. 
Applying the previous reasoning to $Y$ also brings contradiction, and we can conclude. 
\end{proof}

\section{Three more movies}

\label{app:more_movies}

The data from \citet{simoiu2019studying} deal with $20$ movies, and only three were displayed in the main paper. We look here at three more movies: Logan, The Space Between Us, and Smurfs, The Lost Village. As illustrated in Figure \ref{fig:crowds2}, the results are again in line with our theory. An interesting feature of the Smurfs movie is that  while it is in the nonconvex part of the loss, it is close to the inflexion point, since $|y - \bar{y}_\infty| \approx 13.8$ is closer to $10$ than any of the five other movies. This perhaps explains while the Welsch and Geman-McClure losses are initially decreasing, and eventually increasing.

Convergence is in general much slower when $\bar{y}_\infty$ is close to the inflection point of $10$, than when raters are either very accurate (like for The Space Between Us), or clearly missed the target (John Wick). This can be explained by the delta method expansion of Lemma \ref{lem:delta}, that implies 
\begin{equation}
    \mathbb{E} [L(\ybar_k)] = L(\ybar_\infty) + \frac{L''(\ybar_\infty) \text{Var}(\yhat_1)}{2 K} + \mathcal{O}\left(\frac{1}{K^{3/2}}\right),
\end{equation}
which means that convergence will be slower when the second derivative is small, \emph{i.e.} near the inflexion point.

\begin{figure}
\centering
\includegraphics[width = \columnwidth]{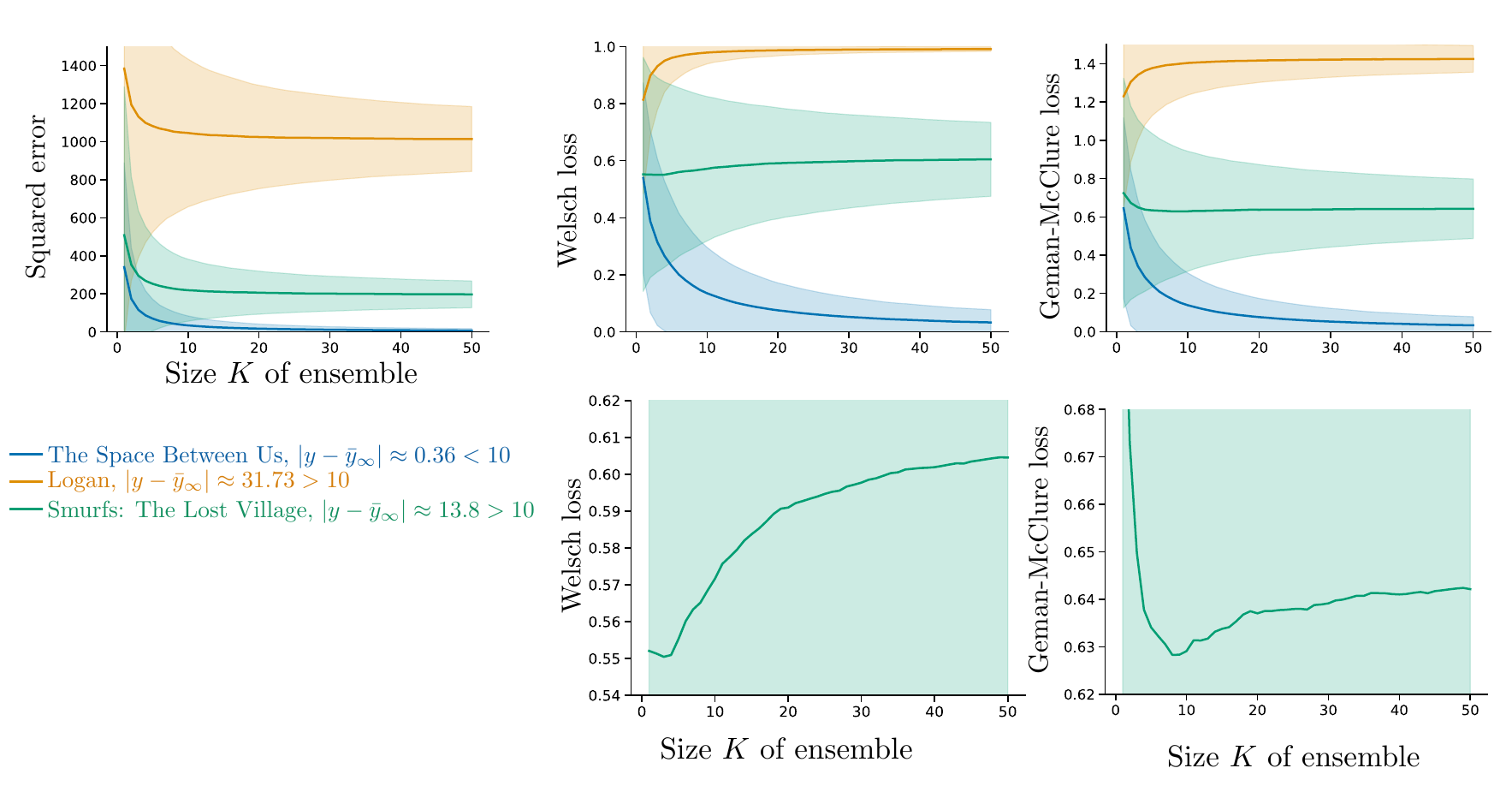}
\caption{\label{fig:crowds2}Evolution of the error at predicting the ratings of three movies as the size of the crowd grows (mean and standard deviation over $10,000$ repetitions). \emph{(Left)} The squared error is decreasing for all three movies, as predicted by Theorem \ref{th:exch}. \emph{(Middle and Right)} The Welsh and Geman-McClure losses are clearly increasing for Logan, and decreasing for The Space Between us, in line with Theorem \ref{th:smooth-ncvx}. The Smurfs are not monochromatic as we expected: as displayed in the two zoom-ins of the bottom panels, the nonconvex losses are decreasing at first, then end-up increasing. This remains consistent with Theorem \ref{th:smooth-ncvx}, that merely ensures that these two loss will be \emph{eventually} monotonic.}
\end{figure}